\documentclass[sigconf]{acmart}

\usepackage{graphicx}
\usepackage{subcaption}
\usepackage{caption}
\usepackage{multirow}
\usepackage{amsmath}
\usepackage{amsthm}

\usepackage{algorithm}
\usepackage{algpseudocode}

\newtheorem{definition}{Definition}
\newtheorem{lemma}{Lemma}
\newtheorem{theorem}{Theorem}
\newtheorem{remark}{Remark}
\begin{document}

\title{Gradient Inversion in Federated Reinforcement Learning}

\author{Shenghong He}
\authornote{Both authors contributed equally to this research.}
\email{tianleh6@gmail.com}
\orcid{1234-5678-9012}
\author{G.K.M. Anle}
\authornotemark[1]
\affiliation{%
  \institution{Distributed Computing Research Center}
  \city{Dublin}
  \state{Ohio}
  \country{USA}
}









\begin{abstract}
Federated reinforcement learning (FRL) enables distributed learning of optimal policies while preserving local data privacy through gradient sharing.
However, FRL faces the risk of data privacy leaks, where attackers exploit shared gradients to reconstruct local training data.
Compared to traditional supervised federated learning,  successful reconstruction in FRL requires the generated data not only to match the shared gradients but also to align with real transition dynamics of the environment (i.e., aligning with the real data transition distribution).
To address this issue, we propose a novel attack method called Regularization Gradient Inversion Attack (RGIA), which enforces prior-knowledge-based regularization on states, rewards, and transition dynamics during the optimization process to ensure that the reconstructed data remain close to the true transition distribution.
Theoretically, we prove that the prior-knowledge-based regularization term narrows the solution space from a broad set containing spurious solutions to a constrained subset that satisfies both gradient matching and true transition dynamics.
Extensive experiments on control tasks and autonomous driving tasks demonstrate that RGIA can effectively constrain reconstructed data transition distributions and thus successfully reconstruct local private data.
\end{abstract}





\maketitle

\section{Introduction}


In real-world scenarios, the training data of machine learning~\cite{DBLP:UppuluriPMKC25,DBLP:LiuYJLL24} is often scattered across various organizations, companies, and individual devices, which leads to the formation of isolated data silos.
For instance, hospitals possess distinct patient datasets, while banks maintain separate financial transaction records. 
Business competition, regulatory constraints, and privacy concerns further prevent the integration of training data.
This issue of data isolation severely hinders the applications of machine learning, particularly in reinforcement learning (RL) that requires collecting large amounts of data from diverse sources to learn effective policies.


To address the issue of data isolation in RL, recent work~\cite{DBLP:MartinBMORB25,DBLP:YouDLGWS25} extends federated learning to RL, i.e., Federated RL (FRL), which enables agents to collaborate on RL tasks while preserving data privacy.
Specifically, FRL includes a central agent and multiple local agents, where the local agents are trained with private data, and the central agent aggregates the gradients or losses from the local agents to update the model parameters.
This approach allows the central agent to learn a globally optimal policy without accessing local data, eliminating both the need for centralized data collection and the risk of data leakage.
Although having these advantages, FRL still faces significant security threats.
In particular, recent studies~\cite{DBLP:abs-2103-06473,DBLP:FangWG25} have shown that FRL is highly susceptible to adversarial attacks, where malicious perturbations are injected into the gradients of local agents to degrade the overall performance of the FRL system.
Furthermore, some studies~\cite{DBLP:abs-2303-02725} investigate environment poisoning attacks, which assume attackers can manipulate a subset of agents and poison the system by perturbing local observations during the policy training process.

In contrast to existing works, we study gradient inversion attacks in FRL, where attackers attempt to reconstruct the private training data of the agent from uploaded gradients, causing privacy leakage.
Unlike supervised learning gradient inversion attacks~\cite{DBLP:ZhuLH19,DBLP:ZhangXGYZ25}, gradient inversion in FRL faces a unique \textit{pseudo-solution problem}: reconstructed data may match the values of gradients while violating the true data transition distribution (i.e., violating the environmental dynamics).
Specifically, the gradients of the value function in FRL is composed of two components: (1) the temporal-difference (TD) error term, and (2) the value function gradients with respect to network parameters. 
Since the TD target depends jointly on reward $r$ and next state $s'$, and the value function gradients are determined by current state $s$, different $(s, a, r, s')$ tuples can yield identical gradient values, leading to ambiguity in reconstruction.
This ambiguity thus causes the generation of different reconstructed tuples from the same gradient during the data reconstruction process, among which some tuples may deviate from the true transition distribution, making a state $\tilde{s}'$ may be unreachable from $\tilde{s}$ under action $\tilde{a}$.
To address the above issue, we propose Regularization Gradient Inversion Attack (RGIA), which extends existing gradient inversion techniques by incorporating data distribution constraints to mitigate convergence toward pseudo-solutions.
Specifically, RGIA first samples candidate variables (state, action, reward, and next state) from a latent space (e.g., a normal distribution) to compute estimated gradients.
Then, RGIA iteratively optimizes the sampled variables to minimize the discrepancies between the estimated gradients and the real shared gradients, aiming to recover the original training data.
To prevent convergence to pseudo-solutions, RGIA integrates three prior-knowledge-based regularization terms into the optimization process: (1) \textbf{the state regularization term} constrains reconstructed states to remain close to the true state distribution via a Euclidean norm constraint, thereby avoiding the out-of-distribution problem; (2) \textbf{the reward range regularization term} constrains reconstructed rewards to lie within plausible bounds, excluding gradient-matching solutions that violate real reward limits; and (3) \textbf{the dynamics consistency regularization term} constrains reconstructed data to satisfy transition dynamics of the environment by minimizing the $L_2$ norm between predicted and true next states.
By embedding these semantic constraints, RGIA guides the optimization process toward a more feasible and trustworthy solution space. 
Our main contributions are as follows:
\begin{itemize}
    \item We conduct the first study on data reconstruction attacks in FRL, providing new insights into data security risks in the FRL setting.
    \item We propose RGIA, a novel attack method that addresses the pseudo-solution problem by incorporating prior-knowledge-based regularizations, and theoretically show that these constraints effectively reduce the feasible solution space.
    \item We validate RGIA through extensive experiments, demonstrating its capability in accurately reconstructing training data from shared gradients of local agents.
\end{itemize}

\section{Related Work}
\subsection{Gradient Inversion Attacks}
Federated Learning (FL)~\cite{DBLP:0004QX24,DBLP:NguyenB25} is a distributed machine learning paradigm that enables multiple participants to collaboratively train a model without sharing private data.
Although private data remains local, the shared gradients in FL may expose sensitive information.
Gradient inversion attacks~\cite{DBLP:LiangLZLZ23,DBLP:ChangZ24} exemplify this risk of privacy leakage, where attackers can hijack the gradients uploaded by a client and use this gradient information to inversely reconstruct the original data.
The existing research on gradient inversion attacks can be broadly categorized into the following classes: optimization-based attacks, analytics-based attacks, and generation-based attacks.

Optimization-based attacks~\cite{DBLP:ZhuLH19,DBLP:abs-2001-02610} formalize gradient matching as an optimization problem and reconstruct the training data by minimizing the difference between the true gradients and the fake gradients.
For example, Zhu et al.~\cite{DBLP:ZhuLH19} demonstrate near-perfect training data reconstruction in shallow networks via gradient optimization.
In contrast, analytics-based attacks~\cite{DBLP:HuangGSLA21,DBLP:KaissisZPRUTLMJ21} attempt to directly infer information about private data from gradients, rather than through complex iterative optimization.
Moreover, generation-based attacks~\cite{DBLP:FangCWWX23,DBLP:Shi00Z00L25} leverage generative models (e.g., GANs~\cite{DBLP:GoodfellowPMXWOCB14}) as prior knowledge to facilitate data reconstruction and generate more realistic and distribution-compliant training data.
Although these approaches can reconstruct data matching the shared gradients, they can not be readily applied in RL settings, where the resulting (state, action, reward, next state) tuples can violate the true environment dynamics. 
This inconsistency therefore causes the invalidity of the reconstructed data and ultimately the failure of attacks.

\subsection{Adversarial Attacks in FRL}
FRL~\cite{DBLP:MartinBMORB25,DBLP:YouDLGWS25} is an emerging learning paradigm that combines the privacy protection advantages of FL with the sequential decision-making capabilities of RL.
However, similar to conventional FL, FRL faces a series of security issues.
Specifically, Ma et al.~\cite{DBLP:abs-2303-02725} present a theoretical framework that formulates environmental poisoning in FRL as an optimization problem.
In this framework, attackers can manipulate the behaviour of agents by perturbing the observations of agents during local policy training and updating.
Similarly, Anwar et al.~\cite{DBLP:abs-2103-06473} propose an adversarial attack method called AdAMInG, which carefully tunes adversarial parameters to reduce the clean information gain during model aggregation.
This attack disrupts the optimization of the center agent and ultimately degrades overall system performance.
Unlike previous approaches, Fang et al.~\cite{DBLP:FangWG25} propose the normalized attack, which reduces the performance of FRL systems by maximizing the angle deviation of the aggregation strategy update before and after the attack.
In contrast to existing works, we explore gradient inversion attacks in FRL to reconstruct the private training data (i.e., state, action, reward, and next state) of the agent from uploaded local policy gradients. 
To alleviate the problem of pseudo-solutions in reconstructed data, we introduce prior-knowledge-based regularization terms that reduce the range of reconstruction data solution space by imposing a distribution constraint on the minimization of gradient loss.

\section{Preliminary}
\noindent\textbf{MDP and RL}.
The Markov Decision Process (MDP) is a widely used decision-making model for RL, formally represented as $\mathcal{M}=(\mathcal{S},\mathcal{A},\mathcal{R},\mathcal{P},\mu_0)$, where $\mathcal{S}$ is the state space, $\mathcal{A}$ is the action space, $\mathcal{R}:\mathcal{S}\times\mathcal{A}\to \mathbb{R}$ is the reward function, $\mathcal{P}:\mathcal{S}\times\mathcal{A}\to\Delta(\mathcal{S})$ is the transition model (with $\Delta(\mathcal{S})$ denoting state probability distributions), and $\mu_0$ is the initial state distribution.
Under this setup, at timestep $t$, the agent observes state $s_t \in \mathcal{S}$ and selects action $a_t \in \mathcal{A}$ according to a policy $\pi$, which maps states to action distributions: $\pi:\mathcal{S}\mapsto\Delta(\mathcal{A})$. 
The value function $V^{\pi}(s)$ quantifies the expected cumulative reward under $\pi$ from state $s$: $V^{\pi}(s)=\mathbb{E}\left[\sum_{t=0}^T\mathcal{R}(s_t,a_t)\mid \pi,s_0=s\right], \forall s \in \mathcal{S}$.
Based on the value function, the goal of RL is to find the optimal policy $\pi^*$ that maximizes the expected value from initial states: $\pi^*=\mathop{\text{argmax}}_{\pi} \mathbb{E}_{s_0\sim\mu_0}\left[V^{\pi}\left(s_0\right)\right]$.


\begin{figure}[t]
    \centering
    \includegraphics[width=0.73\linewidth]{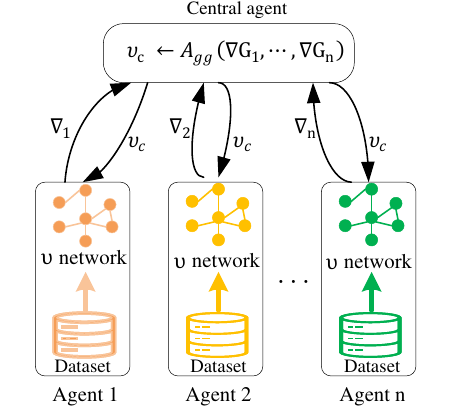}
    \caption{The FRL learning framework. $\upsilon$ network is the value function network.}
    \label{fig:FRL_struct}
\end{figure}

$\quad$
\\
\noindent \textbf{FRL}.
FRL enables multiple agents to collaboratively learn a globally optimal policy while preserving data privacy.
As illustrated in Figure~\ref{fig:FRL_struct}, each local agent $i$ computes the local value function network gradient \(\nabla_i\) and sends it to the central server.
The server aggregates these gradients through a function $A_{gg}$ (e.g., averaging):
$\nabla_{\text{c}}=A_{gg}(\nabla_1,\cdots,\nabla_n)$, where $n$ denotes the total number of agents.
Based on the aggregated gradients $\nabla_{\text{c}}$, the central value function network $\upsilon_c$ is updated via gradient ascent:
\begin{eqnarray}
    \upsilon_{\text{c}} \leftarrow \upsilon_{\text{c}} + \alpha \nabla_{\text{c}},
\end{eqnarray}
where \(\alpha\) is the learning rate.
Then, the updated network $\upsilon_c$ (or its parameters) is sent to all agents to continue local learning.
The loop is executed until either the model converges or a predefined number of iterations is reached.


$\quad$
\\
\textbf{Gradient inversion attacks}.
Gradient inversion attacks aim to reconstruct private input data by leveraging the gradients exposed from the target model.
Typically, attackers initialize reconstructed data ($\mathbf{x}_{\text{rec}}$) and labels ($\mathbf{y}_{\text{rec}}$) randomly, and then iteratively optimizes both to minimize the difference between the fake gradients ($\nabla_{\theta} \mathcal{L}(\mathbf{x}_{\text{rec}}, \mathbf{y}_{\text{rec}})$) and the true gradients ($\nabla_{\theta} \mathcal{L}(\mathbf{x}, \mathbf{y})$), computed from original training data.
Formally, this is formalized by minimizing the Euclidean distance between the two gradients:
\begin{eqnarray}
    \min_{\mathbf{x}_{\text{rec}}, \mathbf{y}_{\text{rec}}} \|\nabla _\theta\mathcal{L}\left(\mathbf{x}_{\text{rec}}, \mathbf{y}_{\text{rec}}\right)-\nabla_\theta \mathcal{L}(x,y)\|^2+\lambda \|\mathbf{x}_{\text{rec}}\|^2,
\end{eqnarray}
where $\lambda$ is a regularization parameter used to prevent overfitting.

$\quad$
\\
\noindent\textbf{Attack knowledge}.
In FRL, attackers lack access to local agent training data but can access the central server to obtain uploaded gradients and possess a priori knowledge of part of the training data.
This assumption is common in prior work~\cite{DBLP:0004QX24,DBLP:NguyenB25} and highly realistic.
For example, in FRL applied to clinical decision-making, attackers may not have direct access to the private medical records of the hospital. 
However, by utilizing publicly available datasets or partial case information from past patients, attackers can infer an approximate distribution of the medical data.
In our setting, we limit the proportion of such prior data to only 0.3$\%$ of the overall training set.

$\quad$
\\
\textbf{Attack ability}.
During the model update phase in FRL, attackers can hijack one or more uploaded gradients to reconstruct training data.
Specifically, attackers employ a shadow model that is identical to the model of the central server (i.e., the value function). 
This shadow model is used to generate fake gradients $\nabla \mathcal{L}_{fake}$, which are computed based on randomly sampled state, action, and reward variables.
Attackers then reconstruct the training data by iteratively optimizing these sampled variables to minimize the discrepancy between the fake gradients and the real uploaded gradients:
\begin{equation}
    \begin{split}
        \mathbf{s}'', \mathbf{a}'',\mathbf{r}''&=\underset{\mathbf{s}, \mathbf{a},\mathbf{r}}{\arg \min }\left\|\nabla \mathcal{L}_{fake}-\nabla \mathcal{L}_{real} \right\|^{2}\\
        &=\underset{\mathbf{s}, \mathbf{a},\mathbf{r}}{\arg \min }\left\|\frac{\partial \ell\left(F\left(\mathbf{s}',\mathbf{a}', w\right), \mathbf{r}'\right)}{\partial w}-\nabla \mathcal{L}_{real}\right\|^{2},
    \end{split}
\end{equation}
where $w$ is the parameter of the model.

\section{method}
In this section, we describe how RGIA reconstructs the training states, actions, and rewards from shared gradients, as illustrated in Algorithm~\ref{alg:rgia}.
Attackers first hijack the gradient ($\nabla_\theta \mathcal{L}_{real}$)  from the local agent to the central agent.
Next, fake samples are generated from the normal distribution (Line 2), and the corresponding fake gradients are computed (Line 4).
Subsequently, the prior-knowledge-based regularization is incorporated into the loss function for both the real gradients and the fake gradients, and the fake examples are iteratively updated via gradient descent (Line 5 $\sim$ 13).
Finally, by minimizing this loss function, attackers reconstruct the state, action, reward and next state.

\begin{algorithm}[t]
\caption{RGIA}
\label{alg:rgia}
\begin{algorithmic}[1]
\Require real sample gradient $\nabla_\theta \mathcal{L}_{real}$, learning rate $\eta$.

\State \textbf{Initialize:}
\State \quad Initialize fake training data samples $\tilde{s}, \tilde{a}, \tilde{r}, \tilde{s}'$ from $\mathcal{N}(0, 1)$.

\For{$i = 1$ to max\_iterations}
    \State Compute the fake gradient $\nabla_\theta \mathcal{L}_{fake}$ with respect to the Q-network parameters.
    \State \textbf{Define the combined loss function $\mathcal{L}$:}
    \State \quad $\mathcal{R}_{\text{total}}(\tilde{s}, \tilde{a}, \tilde{r}, \tilde{s}') = \alpha \mathcal{R}_s + \beta \mathcal{R}_r + \gamma \mathcal{R}_{f}$,
    \State \quad $\mathcal{L} = \min_{\tilde{s}, \tilde{a}, \tilde{r}, \tilde{s}'} \lVert \nabla_\theta L_{fake} - \nabla_\theta L_{real} \rVert^2 + \lambda \mathcal{R}_{\text{total}}(\tilde{s}, \tilde{a}, \tilde{r}, \tilde{s}')$.
    \State \textbf{Update the fake training data samples using gradient descent:}
    \State \quad $\tilde{s}_{i+1} = \tilde{s}_i - \eta \nabla_{\tilde{s}} \mathcal{L}$,
    \State \quad $\tilde{a}_{i+1} = \tilde{a}_i - \eta \nabla_{\tilde{a}} \mathcal{L}$,
    \State \quad $\tilde{r}_{i+1} = \tilde{r}_i - \eta \nabla_{\tilde{r}} \mathcal{L}$,
    \State \quad $\tilde{s}'_{i+1} = \tilde{s}'_i - \eta \nabla_{\tilde{s}'} \mathcal{L}$.
\EndFor

\State \textbf{Return:} $(\tilde{s}, \tilde{a}, \tilde{r}, \tilde{s}')$.

\end{algorithmic}
\end{algorithm}

\subsection{Gradient Derivation Training Data}
The goal of RGIA is to reconstruct the FRL training data samples $(s, a, r, s')$  from the gradients $\nabla_\theta \mathcal{L}$ obtained during training.
We first recall the standard value function loss:
\begin{equation}
    \mathcal{L}= \frac{1}{2}(Q_\theta(s, a) - y)^2.
\end{equation}
Here, $Q_\theta(s, a)$ is the predicted value of the current network for the current state and action, and $y$ is the target Q value, which can be expressed as $y= r + \gamma \max_{a'} Q_{\theta^-}(s', a')$, where $Q_{\theta^-}$ is the target Q-network.
To recover $(s, a, r, s')$, RGIA initializes learnable variables $(\tilde{s}, \tilde{a}, \tilde{r}, \tilde{s}') \sim \mathcal{N}(0, 1)$ and constructs a fake loss $\mathcal{L}_{\text{fake}}$ by:
\begin{equation}
    \tilde{y} = \tilde{r} + \gamma \max_{a'} Q_{\theta^-}(\tilde{s}', a'),
\end{equation}
\begin{equation}
    \mathcal{L}_{\text{fake}} = \frac{1}{2}(Q_\theta(\tilde{s}, \tilde{a}) - \tilde{y})^2.
\end{equation}

RGIA computes the fake gradients by taking the derivative of the fake loss with respect to the Q-network parameters, and then minimizes the discrepancy between the fake gradients and the real gradients as follows:
\begin{equation}
\label{eq:min_gradient}
    \mathcal{L}=\min_{\tilde{s}, \tilde{a}, \tilde{r}, \tilde{s}'} \left\| \nabla_\theta \mathcal{L}_{\text{fake}} - \nabla_\theta \mathcal{L}_{\text{real}} \right\|^2.
\end{equation}
To optimize Eq.~(\ref{eq:min_gradient}), RGIA iteratively updates the fake data so that the fake gradients match the real ones closely.
This optimization proceeds via gradient descent as follows:
\begin{equation}
\begin{cases} 
\tilde{s}_{i+1}=\tilde{s}_i-\eta \nabla_{\tilde{s}}  \mathcal{L} &  \\
\tilde{a}_{i+1}=\tilde{a}_i-\eta \nabla_{\tilde{a}}  \mathcal{L} &\\
\tilde{r}_{i+1}=\tilde{r}_i-\eta \nabla_{\tilde{r}}  \mathcal{L} & \\
\tilde{s'}_{i+1}=\tilde{s'}_i-\eta \nabla_{\tilde{s'}}  \mathcal{L}, 
\end{cases}
\end{equation}
where $\eta$ is the learning step size and $i$ the number of iterations.

\subsection{Prior-knowledge-based Regularization Terms}
Although gradient derivation can minimize the difference between the fake and real gradients to reconstruct training samples, it faces a pseudo-solution issue.
Specifically, the same gradient can reconstruct different  $(s, a, r, s') $ tuples, which may not conform to the true data transition distribution.
To address this issue, we introduce the prior-knowledge-based regularization term into the gradient matching optimization process, in order to strengthen the distributional constraints on the reconstructed data and guide the optimization towards more trustworthy reconstruction results.
Specifically, we construct three distinct regularization terms, which correspond to states, rewards, and transition dynamics terms, respectively.

The first term is a state regularization, which forces the reconstructed states to be close to the true distribution $\mathcal{S}$: $\mathcal{R}_s = \|\tilde{s} - \mu_s\|^2$, where $\mu_s$ denotes the empirical state distribution. 
By enforcing adherence of reconstructed states to the true distribution, this term can prevent the generation of anomalous states with excessively large distribution deviations.
The second term is a reward range regularization that acts as a soft penalty on reconstructed rewards, defined as $\mathcal{R}_r = \left(\text{ReLU}(\tilde{r} - r_{\max})\right)^2 + \left(\text{ReLU}(r_{\min} - \tilde{r})\right)^2$.
This term ensures that the reconstructed reward value is within a reasonable range and prevents the generation of reward values that can match the true gradients but violate range values.
The last term is a dynamics consistency regularization, which ensures that the $(s, a, s')$ triples are feasible under environmental dynamics.
To achieve this, a state transition model $f:\mathcal{S} \times \mathcal{A} \rightarrow \mathcal{S}$ is pre-trained and incorporated into the optimization via the loss  $\mathcal{R}_{f} = \|f(\tilde{s}, \tilde{a}) - \tilde{s}'\|^2$. 
This regularization encourages the reconstructed next state $\tilde{s}'$ to align with the predicted transition from $\tilde{s}$ and $\tilde{a}$, ensuring that the generated trajectory segment exists in the real environment.

Finally, we define the combined regularization term as \\
$\mathcal{R}_{\text{total}}(\tilde{s}, \tilde{a}, \tilde{r}, \tilde{s}')  = \alpha \mathcal{R}_s + \beta \mathcal{R}_r + \gamma \mathcal{R}_{f}$, 
where $\alpha$, $\beta$, and $\gamma$ are the weighting coefficients for the state, reward, and dynamics regularization terms, respectively. 
Based on this definition, we rewrite Eq.~(\ref{eq:min_gradient}) into the following form:
\begin{equation}
\label{eq:tot_contranst}
    \mathcal{L}=\min_{\tilde{s}, \tilde{a}, \tilde{r}, \tilde{s}'} \left\| \nabla_\theta \mathcal{L}_{\text{fake}} - \nabla_\theta \mathcal{L}_{\text{real}} \right\|^2+ \lambda \mathcal{R}_{\text{total}}(\tilde{s}, \tilde{a}, \tilde{r}, \tilde{s}'),
\end{equation}
where $\lambda > 0$ is a balancing hyperparameter that controls the weight of $\mathcal{R}_{\text{total}}$ in the overall objective.
By introducing these environment-related prior constraints, we can effectively restrict the solution space to more plausible reconstructions, which in turn makes the gradient inversion process more tractable and aligned with the underlying environment dynamics.

\subsection{Regularization Analysis}
To validate the effectiveness of the proposed method, this section conducts a theoretical analysis to demonstrate that RGIA achieves the desired performance.
We first formally define the relevant symbols and problems.
\begin{definition}[Sample and parameter space]
     Let $\mathcal{S} \subseteq \mathbb{R}^{d_s}$ denote the state space, $\mathcal{A}$ the action space, and $\mathcal{R} \subseteq \mathbb{R}^{d_r}$ the reward space.
     A sample to be reconstructed is represented by the tuple $x = (s, a, r, s') \in \mathcal{S} \times \mathcal{A} \times \mathcal{R} \times \mathcal{S}$, which consists of a state, action, reward, and next state.  Additionally,  $\theta \in \Theta \subseteq \mathbb{R}^{d_\theta}$ represents the parameters of the neural network model.
\end{definition}

\begin{definition}[Gradient inverse problem]
    Given an observed gradient \(g_{\text{obs}} = \nabla_\theta L(x_{\text{real}}; \theta)\) computed from real data \(x_{\text{real}}\), the goal of gradient inversion reconstruction is to find a reconstructed sample \(\tilde{x} = (\tilde{s}, \tilde{a}, \tilde{r}, \tilde{s}')\) that minimizes the following objective function:
$F(\tilde{x}) = \| \nabla_\theta L(\tilde{x}; \theta) - g_{\text{obs}} \|^2 \quad (P1)$.
\end{definition}

\begin{definition}[Original solution space]
\label{def:ossp}
    The original solution space $X_{orig}$ for problem (P1) is defined as the set of all samples $\tilde{x}$ satisfying $X_{orig} = \{ \tilde{x} \mid \| \nabla_{\theta} L(\tilde{x}; \theta) - g_{obs} \|^2 = 0 \}$.
\end{definition}
Subsequently, we analyze the constraining effect of each regularization term on the solution space through the following remarks and lemma.

\begin{remark}[State space constraints]
\label{rem:remark1}
    The regularization term \( R_s = | \tilde{s} - \mu_s |^2 \) encourages the optimization process to find a state \(\tilde{s}\) near \(\mu_s\). Intuitively, the term penalizes large deviations from \(\mu_s\), so optimal solutions \(\tilde{s}^*\) tend to lie within a neighborhood of \(\mu_s\). The size of this neighborhood depends on the trade-off between the primary objective (e.g., gradient matching error) and the regularization weight \(\alpha\): a larger \(\alpha\) shrinks the neighborhood, while a larger objective value may allow a larger deviation.
\end{remark}

    



\begin{lemma}[Reward space constraints]
\label{lem:lemma2}
    Given the optimization problem:
    $\min_{\tilde{x}} J(\tilde{x}) = F(\tilde{x}) + \lambda \mathcal{R}_r(\tilde{r}), \quad \lambda > 0$,
    where $F(\tilde{x})$ is the primary objective function and $\mathcal{R}_r$ is the regularization term defined as:
    $\mathcal{R}_r(\tilde{r}) = \left(\text{ReLU}(\tilde{r} - r_{\max})\right)^2 + \left(\text{ReLU}(r_{\min} - \tilde{r})\right)^2$, then for any optimal solution $\tilde{x}^*$, the reward component $\tilde{r}^*$ must satisfy $\tilde{r}^* \in [r_{\min}, r_{\max}]$. 
\end{lemma}
Proof can be found in Appendix~\ref{supsec:lemma_proof}. Lemma~\ref{lem:lemma2} indicates that the regularization term $\mathcal{R}_r = \left(\text{ReLU}(\tilde{r} - r_{\max})\right)^2 + \left(\text{ReLU}(r_{\min} - \tilde{r})\right)^2$ imposes a constraint on the solution space: the reward component $\tilde{r}^*$ of any optimal solution must lie within $[r_{\min}, r_{\max}]$. Solutions violating this interval incur a significant penalty.



\begin{remark}
The regularization term $\mathcal{R}_{\text{f}} = \| f(\tilde{s}, \tilde{a}) - \tilde{s}' \|^2$ serves to enforce dynamical consistency. By minimizing this term, the solution $x = (\tilde{s}, \tilde{a}, \tilde{s}')$ is constrained to the learned dynamical manifold $M_f$, ensuring that the predicted next state $\tilde{s}'$ is consistent with the state-action pair as defined by the transition model $f$.
\end{remark}

Based on the above remarks and lemma, we can obtain the following theorem.
\begin{theorem}[Regularization for compressing the solution space]
\label{theorem:one}
    Let $\mathcal{X}_{\text{reg}}^*$ denote the set of optimal solutions to the regularized problem.
    By introducing the regularization term $\mathcal{R}_{\text{total}}(x)$, the solution space is effectively compressed $\mathcal{X}^*_\text{reg} \subseteq \mathcal{X}_\text{orig} \cap \mathcal{C}_\text{prior}$, where $\mathcal{C}_{\text{prior}}$ denotes the constraint set defined by the prior-informed regularization.
\end{theorem}

Detailed proof is provided in Appendix~\ref{supsec:theorem_proof}.
Theorem~\ref{theorem:one} guarantees that the optimal solution $\tilde{x}^*$ lies in the intersection of the two constraint sets, i.e., $\tilde{x}^* \in \mathcal{X}_{\text{orig}}^\epsilon \cap \mathcal{C}_{\text{prior}}$.
This intersection reflects a well-behaved subset of the original search space, where solutions satisfy both gradient-matching requirements and physical plausibility constraints.
Consequently, by exploiting this subset, RGIA effectively narrows the solution space, filtering out spurious solutions that violate these constraints and focusing on physically plausible candidates.

\section{Experiments}
This section presents the experimental setup and evaluates the attack performance of RGIA.

\subsection{Experiment Setup}
\noindent\textbf{Experiment environments}.
To comprehensively evaluate the data reconstruction capabilities of RGIA, we conduct experiments across a diverse set of environments that vary in both action space type and state dimensionality. Specifically, we include: a high-dimensional (i.e., image input) discrete action environment (Atari), a high-dimensional continuous action environment (Car Racing), a low-dimensional (i.e., One-dimensional vector input) continuous action environment (MuJoCo), and a low-dimensional discrete action environment (Frozen Lake). 
Detailed descriptions of these environments are provided in Appendix~\ref{supsec:envion_infor}.


$\quad$
\\
\noindent\textbf{RL algorithms}.
To effectively evaluate the attack performance of RGIA on different FRL algorithms, we select two representative methods: the online algorithm Actor-Critic (AC)~\cite{DBLP:MnihBMGLHSK16} and the offline algorithm BCQ~\cite{DBLP:FujimotoMP19}. 
Specifically, we construct an FRL framework consisting of three local agents and one central agent.
In this setup, each local agent uses the AC or BCQ algorithm to learning a policy. 
During training, the local agents compute and upload value function gradients to the central agent, which aggregates these gradients to update a shared central value function model. 
Finally, the updated model parameters are broadcast back to all local agents to initiate the next round of learning.


$\quad$
\\
\noindent\textbf{Evaluation metrics}.
To evaluate gradient-based RL inversion attacks, we use distinct metrics based on the input type.
For image state inputs, we assess the similarity between reconstructed and original images using PSNR~\cite{kunt1985second} and SSIM~\cite{wang2004image}.
For low-dimensional state inputs, we compute the mean squared error (MSE) to quantify the difference between the original and reconstructed state features.
For actions, we use recovery accuracy (RA) to measure the reconstruction capability of RGIA, defined as $\frac{n}{m} \times 100\%$, where $n$ is the number of correctly reconstructed actions and $m$ is the total number of evaluated actions.
Since rewards are continuous variables, we also use MSE to evaluate the accuracy of reconstructed rewards.
Moreover, to assess the influence of regularization terms on the reconstructed samples, we introduce consistency metrics, including: Euclidean Distance (ED), Silhouette Score (SS), Covariance Determinant (CD), and Transition Error (TE).
Detailed descriptions of all metrics are provided in Appendix~\ref{supsec:metirc}.

$\quad$
\\
\noindent\textbf{Baselines}.
Existing adversarial attacks in FRL systems primarily focus on reward poisoning and state perturbation. 
In contrast, the goal of RGIA is to reconstruct the training data of each local policy based on the shared gradients in FRL.
To evaluate RGIA, we adapt four supervised learning gradient inversion attacks to the FRL setting: DGL~\cite{DBLP:ZhuLH19}, GradInv~\cite{DBLP:YinMVAKM21}, GIFD~\cite{DBLP:FangCWWX23}, and DFLeak~\cite{DBLP:ZhangXGYZ25}.
This comparison not only highlights the superior performance of RGIA but also reveals the limitations of these adapted attacks in the FRL setting.

\begin{figure*}
    \centering
    \includegraphics[width=0.90\linewidth]{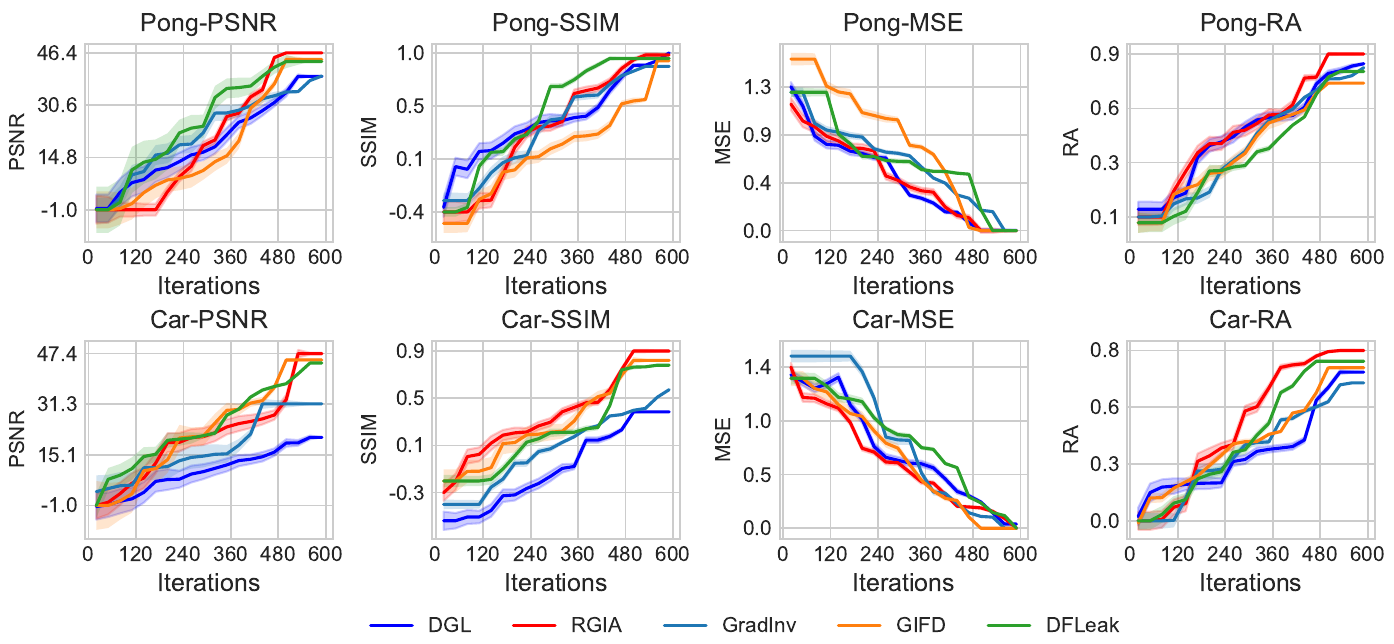}
    \caption{The reconstruction process of training data. The x-axis represents the number of iterations, and the y-axis corresponds to various evaluation metrics.}
    \label{fig:lines}
\end{figure*}

\subsection{Baseline Comparison Results}

\noindent \textbf{High-dimensional environments}.
Figure~\ref{fig:lines} presents a comparison between RGIA and baseline methods (DGL, GrandInv, GIFD, and DFLeak) in online settings using the AC algorithm,
The results show that RGIA consistently outperforms the baselines across these reconstruction metrics.
Specifically, in the Car Racing environment, RGIA achieves better PSNR and SSIM scores, which indicates that the states reconstructed by RGIA are closer to the original inputs.
Similarly, RGIA outperforms other methods in reconstructing actions.
Although reward reconstruction MSE values are comparable across all methods, RGIA maintains slightly lower reward errors.
These results demonstrate that RGIA can effectively constrain the solution space through the regularization of prior knowledge, which leads to more precise reconstruction on the training data.

\begin{figure*}
    \centering
    \includegraphics[width=0.88\linewidth]{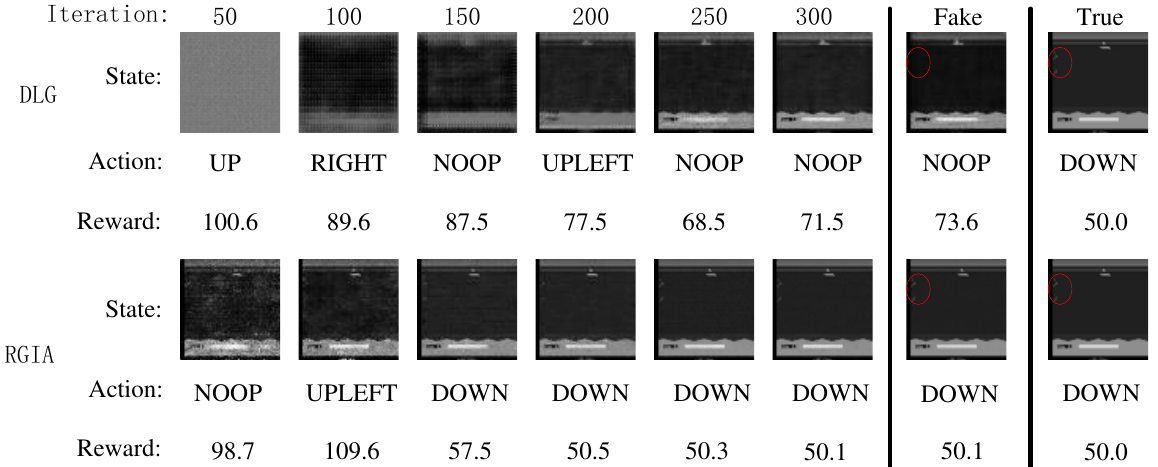}
    \caption{Visualization of the reconstruction process. Fake denotes the data reconstructed by gradient inversion attacks, while True refers to the ground-truth data.}
    \label{fig:reconsted_process}
\end{figure*}

In addition, Figure~\ref{fig:reconsted_process} visualizes the reconstruction process of states, actions, and rewards from gradients of the value function.
The results reveal significant differences between the states generated by DLG and the ground-truth states in the Sequest environment.
As shown in the red ellipse, the ground-truth state includes divers but the DLG reconstructed state lacks divers.
This discrepancy arises because DLG cannot resolve the pseudo-solution problem in FRL, where multiple $(s,a,r,s')$ tuples correspond to the same gradient.
Unlike DLG, RGIA can effectively reconstruct training states by using prior knowledge-based regularization.
Further experimental results and analyses are presented in Appendix~\ref{supsec:high-dim}.

\begin{figure}[]
    \centering
    \begin{subfigure}[b]{0.38\textwidth}
        \centering
        \includegraphics[width=\textwidth]{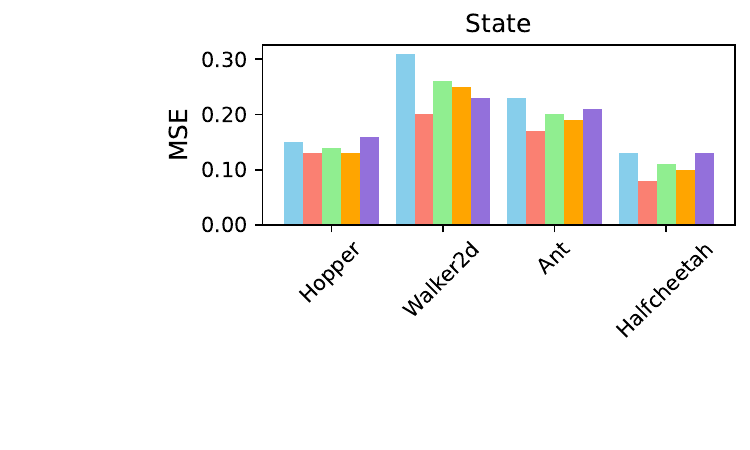}
    \end{subfigure}
    \hfill
    \begin{subfigure}[b]{0.38\textwidth}
        \centering
        \includegraphics[width=\textwidth]{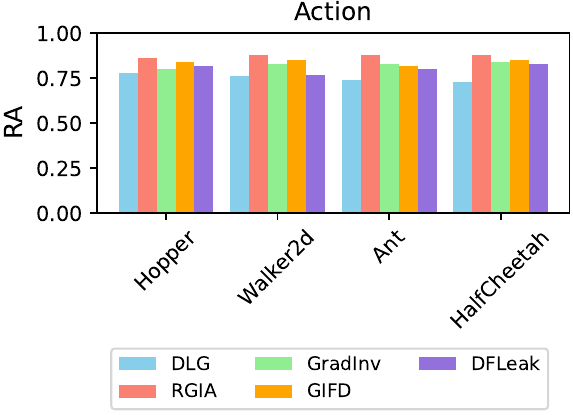}
    \end{subfigure}
    \caption{The comparison results of low-dimensional environments. To ensure consistent visualization of coordinates, the state MSE is scaled up by $10^{-6}$ in the Hopper and Halfcheetah environments, and by $10^{-7}$ in Walker2d and Ant.}
    \label{fig:conti_metric}
\end{figure}

$\quad$
\\
\noindent \textbf{Low-dimensional environments}.
Similarly, in the offline settings with low-dimensional state space, we evaluate RGIA using BCQ, while keeping all other settings unchanged.
Specifically, BCQ performs offline learning using data from MuJoCo environments (Hopper, Walker2D, Ant, and Halfcheetah) and the FrozenLake environment. 
For each environment, the dataset is randomly divided into three subsets to independently train three local agents. 
Subsequently, RGIA, DLG, GrandInv, GIFD, and DFLeak reconstruct the training data from the gradients uploaded by these local agents.

Figure~\ref{fig:conti_metric} shows the reconstruction state and action accuracy (i.e., MSE and RA) of RGIA, DLG, GrandInv, GIFD, and DFLeak in low-dimensional environments, where each method selects 1,000 tuples (i.e.,$(s,a,r,s’)$) for evaluation and reports the average performance.
The results demonstrate that RGIA yields lower reconstruction error (i.e., lower MSE) and higher reconstruction accuracy (RA) than the baseline methods.
While all methods rely on gradient inversion to reconstruct data, DLG, GrandInv, GIFD, and DFLeak fail to address the pseudo-solution problem in FRL. 
As a result, reconstructed data from these methods approximates the true gradient values but deviates from the true data distribution, resulting in higher reconstruction errors.
The experimental results for FrozenLake can be found in Appendix~\ref{supsec:low_dim}.

\subsection{Reconstructed State Distribution  Analysis}
To evaluate whether our proposed regularization can mitigate the pseudo-solution issue by promoting convergence to consistent solutions, we perform a multi-start optimization consistency analysis.
Specifically, we randomly select a set of real transitions $(s, a, r, s')$ and compute their corresponding gradient $ \nabla_\theta \mathcal{L}_{\text{real}}$.
For each gradient, we initialize 10 different random starting variables $(\tilde{s}_0, \tilde{a}_0, \tilde{r}_0, \tilde{s}'_0) \in \mathcal{N}(0,1)$, and independently optimize these variables by utilizing GIA and RGIA, where GIA is the method without regularization terms.

For each method, we collect 100 reconstructed samples and evaluate intra-group consistency in the latent embedding space, with a primary focus on the reconstructed states $\tilde{s}$.
The consistency is evaluated using three metrics: (1) the average pairwise Euclidean distance (ED) between reconstructed states (lower is better); (2) the Silhouette Score (SS) reflecting clustering quality (higher is better); and (3) the covariance determinant (CD) capturing the dispersion of reconstructed states (lower is better).
As shown in Table~\ref{tab:multi_start}, RGIA substantially improves consistency across all metrics compared to GIA.
Specifically, ED is reduced by more than 60$\%$, the SS score is increased by more than twofold, and CD is also significantly decreased.
These results indicate that the reconstructed samples form a tighter and more coherent cluster in the latent state space.


\begin{table}[]
\caption{Multi-start consistency analysis. Regularization significantly reduces dispersion among reconstructed samples.}
\label{tab:multi_start}
\begin{tabular}{l|l|c|c|c}
\hline
Environment               & Method & ED$\downarrow$            & SS$\uparrow$          & CD$\downarrow$              \\ \hline
\multirow{2}{*}{Breakout} & GIA    & 0.284          & 0.32          & 5.12e-2          \\
                          & RGIA   & \textbf{0.108} & \textbf{0.71} & \textbf{1.24e-3} \\ \hline
\multirow{2}{*}{Seaquest} & GIA    & 0.264          & 0.31          & 4.98e-2          \\
                          & RGIA   & \textbf{0.118} & \textbf{0.69} & \textbf{1.13e-3} \\ \hline
\multirow{2}{*}{Ant}      & GIA    & 0.134          & 0.39          & 4.48e-2          \\
                          & RGIA   & \textbf{0.108} & \textbf{0.76} & \textbf{1.03e-3} \\ \hline
\end{tabular}
\end{table}

To further illustrate the effect of regularization, we apply PCA~\cite{pearson1901liii} to project the reconstructed states into a 2D plane. 
As shown in Figure~\ref{fig:PC_sample}, the unregularized reconstructions (red) are scattered widely across the space, whereas RGIA (green) produces samples clustered closely around the true state (blue star).
These results demonstrate that the prior-driven regularization terms not only improve the accuracy of individual reconstructions but also promote semantic consistency across multiple solutions derived from the same gradient. 
\begin{figure}
    \centering
    \includegraphics[width=0.88\linewidth]{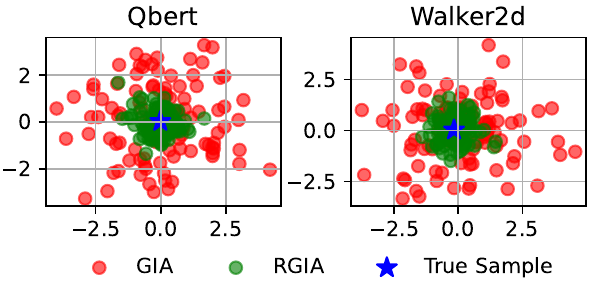}
    \caption{PCA visualization of reconstructed states from 10 random initializations.}
    \label{fig:PC_sample}
\end{figure}

\begin{table}[]
\centering
\caption{Gaussian noise defense results. 'Rewards' denote the cumulative return obtained by the learning policy of FRL within the corresponding environment.}
\label{tab:guassian_nosie}
\begin{tabular}{llcccc}
\hline
Variance & Environment                                                             & Loss$\downarrow$                                                                & Rewards$\uparrow$                                                    & MSE$\downarrow$                                                                & RA$\uparrow$                                                        \\ \hline
1e-1     & \begin{tabular}[c]{@{}l@{}}Hopper\\ Walker2d\\ Halfcheetah\end{tabular} & \begin{tabular}[c]{@{}c@{}}0.72\\ 0.48\\ 0.91\end{tabular}          & \begin{tabular}[c]{@{}c@{}}20.3\\ 34.5\\ 21.6\end{tabular}  & \begin{tabular}[c]{@{}c@{}}0.29\\ 0.25\\ 0.39\end{tabular}          & \begin{tabular}[c]{@{}c@{}}0.02\\ 0.09\\ 0.10\end{tabular} \\ \hline
1e-2     & \begin{tabular}[c]{@{}l@{}}Hopper\\ Walker2d\\ Halfcheetah\end{tabular} & \begin{tabular}[c]{@{}c@{}}0.54\\ 0.28\\ 0.69\end{tabular}          & \begin{tabular}[c]{@{}c@{}}35.\\ 42.6\\ 41.3\end{tabular}   & \begin{tabular}[c]{@{}c@{}}0.21\\ 0.19\\ 0.24\end{tabular}          & \begin{tabular}[c]{@{}c@{}}0.10\\ 0.13\\ 0.14\end{tabular} \\ \hline
1e-3     & \begin{tabular}[c]{@{}l@{}}Hopper\\ Walker2d\\ Halfcheetah\end{tabular} & \begin{tabular}[c]{@{}c@{}}0.028\\ 0.025\\ 0.037\end{tabular}       & \begin{tabular}[c]{@{}c@{}}73.5\\ 89.7\\ 68.4\end{tabular}  & \begin{tabular}[c]{@{}c@{}}0.18e-2\\ 0.15e-1\\ 0.19e-1\end{tabular} & \begin{tabular}[c]{@{}c@{}}0.63\\ 0.71\\ 0.69\end{tabular} \\ \hline
1e-4     & \begin{tabular}[c]{@{}l@{}}Hopper\\ Walker2d\\ Halfcheetah\end{tabular} & \begin{tabular}[c]{@{}c@{}}0.15e-2\\ 0.81e-2\\ 0.69e-2\end{tabular} & \begin{tabular}[c]{@{}c@{}}90.5\\ 101.7\\ 85.6\end{tabular} & \begin{tabular}[c]{@{}c@{}}0.12e-3\\ 0.13e-3\\ 0.18e-2\end{tabular} & \begin{tabular}[c]{@{}c@{}}0.78\\ 0.80\\ 0.76\end{tabular} \\ \hline
1e-5     & \begin{tabular}[c]{@{}l@{}}Hopper\\ Walker2d\\ Halfcheetah\end{tabular} & \begin{tabular}[c]{@{}c@{}}0.04e-2\\ 0.13e-2\\ 0.08e-2\end{tabular} & \begin{tabular}[c]{@{}c@{}}95.5\\ 108.7\\ 89.6\end{tabular} & \begin{tabular}[c]{@{}c@{}}0.10e-5\\ 0.12e-5\\ 0.18e-3\end{tabular} & \begin{tabular}[c]{@{}c@{}}0.81\\ 0.84\\ 0.79\end{tabular} \\ \hline
\end{tabular}
\end{table}

\subsection{RGIA Defense Mechanism Analysis}
\label{sec:defencse_dp}
A potential defense mechanism against RGIA attacks is the differential privacy learning~\cite{DBLP:KianiKDDB25}, which achieves differential privacy  (DP) protection by adding carefully designed noise (typically Gaussian noise or Laplace noise) before uploading gradients.
By leveraging this method, even if attackers know the query results, they still cannot determine whether a specific individual is included in the dataset, and thus data privacy is preserved.
To evaluate the attack capability of RGIA against this DP-based defense, we apply Gaussian noise and Laplace noise with a mean of zero and variances spanning from 1e-1 to 1e-5 to the gradients of the value function.

Table~\ref{tab:guassian_nosie} presents the defense results of RGIA under Gaussian noise (see Appendix~\ref{supsec:dp_def} for Laplace noise). 
These results demonstrate that DP can provide effective protection against RGIA attacks, and the defense effectiveness is not strongly dependent on the type of noise.
However, this defense also degrades policy performance.
Specifically, low-variance noises (1e-4 to 1e-5) minimally impair policy performance but fail to prevent gradient leakage attacks due to insufficient perturbation.
As the variance increases, although the reconstruction accuracy of RGIA progressively declines, it remains capable of recovering training data from gradients.
When the variance is greater than 1e-2, DP noises effectively hinder RGIA from reconstructing training data from the gradients, but they also severely impair the performance of the policy.

In addition, we test two defense mechanisms: homomorphic encryption~\cite{DBLP:LiuKC24} and gradient quantization~\cite{DBLP:OviDRG23} (See Appendix~\ref {supsec:dp_def} for detailed experimental settings).
The experimental results demonstrate that homomorphic encryption and gradient quantization remain insufficient to perfectly defend against RGIA.

\subsection{Leakage Gradient Scale Analysis}
The scale of gradients plays a critical role in determining the quality of reconstructed training data and the optimization difficulty in the RGIA method.
We conduct comparative experiments that evaluate reconstruction performance by using gradients from single sample versus batches of 3, 5, 8, and 10 samples.
Specifically, we select Atari Pong and MuJoCo Hopper, and compute the corresponding gradient values based on a pretrained Q-network.
For each batch size, we sample the corresponding number of tuples $(s,a,r,s'$) from the training data and compute gradients of the Q-network loss function with respect to these samples.
These gradients are treated as the leaked gradients and serve as the optimization objective for the RGIA algorithm.
To effectively measure the differences in data reconstructed from different batch sizes, we employ three metrics: the gradient matching error (GME), the mean squared error (MSE) of reconstruction state, and the reconstruction time (RT).

\begin{table}[]
\caption{The impact of different batches on RGIA.}
\begin{tabular}{lllll}
\hline
Environment & Batch                                                      & GME$\downarrow$                                                                                    & MSE$\downarrow$                                                                                     & RT$\downarrow$ (s)                                                                       \\ \hline
Pong        & \begin{tabular}[c]{@{}l@{}}1\\ 3\\ 5\\ 8\\ 10\end{tabular} & \begin{tabular}[c]{@{}l@{}}8.3e-6\\ 9.3e-5\\ 7.3e-3\\ 3.4e-2\\ 4.2e-1\end{tabular}     & \begin{tabular}[c]{@{}l@{}}0.08e-4\\ 0.27e-3\\ 0.10e-2\\ 0.71e-2\\ 0.044\end{tabular}    & \begin{tabular}[c]{@{}l@{}}44.2\\ 253.3\\ 291.9\\ 362.1\\ 360.8\end{tabular} \\ \hline
Hopper      & \begin{tabular}[c]{@{}l@{}}1\\ 3\\ 5\\ 8\\ 10\end{tabular} & \begin{tabular}[c]{@{}l@{}}5.3e-7\\ 1.38e-4\\ 1.77e-3\\ 1.79e-2\\ 2.09e-2\end{tabular} & \begin{tabular}[c]{@{}l@{}}1.19e-6\\ 3.58e-5\\ 7.64e-5\\ 1.23e-4\\ 2.99e-4\end{tabular} & \begin{tabular}[c]{@{}l@{}}16.6\\ 17.9\\ 30.2\\ 61.9\\ 90.6\end{tabular}     \\ \hline
\end{tabular}
\label{tab:batch_exp}
\end{table}

The experimental results in Table~\ref{tab:batch_exp} indicate that as the leaked gradient batch size increases, the computational time for reconstructing the data increases significantly, and both the gradient matching error and MSE of the reconstructed states exhibit an upward trend.
Specifically, in the Pong environment, the single-sample approach proves to be the optimal reconstruction way, exhibiting significant advantages in terms of GMS, MSE, and RT.
Critically, batch sizes exceeding 3 frequently cause reconstruction failures, producing unusable noise-dominated outputs.
Unlike in the Pong environment, RGIA exhibits better batch reconstruction capability in the Hopper environment. 
For example, when the batch size is 5, both the gradient error and reconstruction error remain relatively low.
This phenomenon may be attributed to the advantage of Hopper having a low-dimensional state space.

\begin{figure*}
    \centering
    \includegraphics[width=0.83\linewidth]{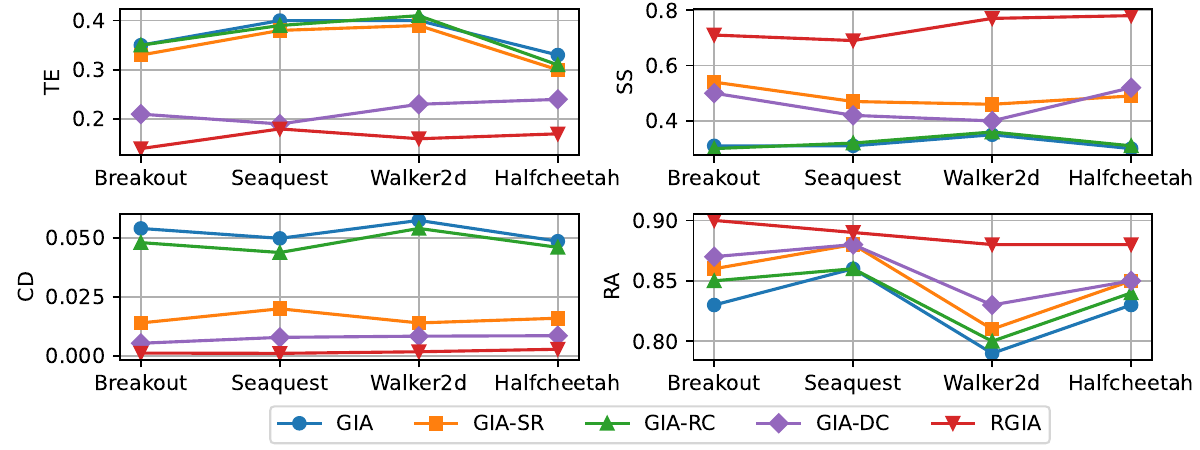}
    \caption{Ablation on regularization terms across different environments. Each regularizer contributes to a different aspect of inversion quality, and their combination yields the most coherent and plausible reconstructions.}
    \label{fig:ablation_multienv}
\end{figure*}

\subsection{Ablation Analysis}
\noindent\textbf{Ablation on regularization terms}.
To understand the individual contributions of the proposed regularization terms, we conduct a comprehensive ablation study by enabling one regularization term at a time while keeping the others disabled.
Specifically, we evaluate the following variants: (1) the baseline without regularization (GIA); (2) GIA with state regularization (GIA-SR); (3) GIA with reward range constraint (GIA-RC); (4) GIA with dynamics consistency  regularization (GIA-DC); and (5) the full method combining all three regularizations (RGIA).
For each configuration, we measure performance differences using four metrics (TE, SS, CD, and RA) across four diverse environments: Breakout, Seaquest, Walker2d, and FrozenLake.

Figure~\ref{fig:ablation_multienv} presents quantitative results averaged over 10 random seeds for each experimental setting. 
As shown in the figure, state regularization term enhances clustering performance (measured by silhouette score, SS) and reduces pairwise dispersion across different environments. 
While the reward range constraint contributes little to the transition error (TE), it plays a crucial role in filtering out implausible rewards, thereby increasing the structural consistency and the success rate of action construction.
Unlike the reward range constraint, the dynamics consistency regularization can significantly reduce transition error and consistently improve silhouette scores.
By integrating three regularizations, RGIA reconstructs semantically coherent and consistent samples, and achieves the best or near-best performance across all evaluation metrics.

$\quad$
\\
\textbf{Hyperparameter sensitivity studies}.
To further evaluate the robustness of RGIA, we conduct a sensitivity analysis by varying the weights of three regularization terms: state regularization term ($\alpha$), reward regularization term ($\beta$), and dynamics consistency regularization term ($\gamma$).
Specifically, we fix the gradient source $\nabla_\theta \mathcal{L}_{\text{real}}$ and perform inversion under different hyperparameter settings.
Inspired by the previous working hyperparameter settings~\cite{DBLP:KostrikovNL22}, we adopt a log-scale grid $\{0.0,0.01,0.1,1.0,10.0\}$ for regularization weights to efficiently explore the effect of these weights across several orders of magnitude. 
To be specific, we vary $\alpha \in \{0.0, 0.01, 0.1,
1.0, 10.0\}$ while fixing $\beta = 1.0$ and $\gamma = 1.0$. 

\begin{table}[]
\centering
\caption{Effect of state regularization weight ($\alpha$) on reconstruction quality.}
\label{tab:alpha_sensitivity}
\begin{tabular}{l|c|c|c}
\hline
         & PSNR$\uparrow$               & SSMIM$\uparrow$              & RA$\uparrow$                 \\
$\alpha$ & Pong/Qbert         & Pong/Qbert         & Pong/Qbert         \\ \hline
0.0      & 40.3/39.8          & 0.93/0.91          & 0.85/0.81          \\
0.01     & 41.7/43.5          & 0.94/0.93          & 0.86/0.84          \\
0.1      & 44.3/49.7          & 0.96/0.96          & 0.89/0.85          \\
1.0      & \textbf{46.3/51.3} & \textbf{0.96/0.98} & \textbf{0.93/0.89} \\
10.0     & 43.4/50.9          & 0.89/0.95          & 0.84/0.85          \\ \hline
\end{tabular}
\end{table}

Table~\ref{tab:alpha_sensitivity} demonstrates that a moderate $\alpha$ setting ($\alpha = 1.0$) optimally enhances RGIA performance across key reconstruction metrics: PSNR, SSIM, and RA.
These improvements appear consistently across both Pong and Qbert tasks.
However, both overly small (e.g., $\alpha = 0.0$) and overly large (e.g., $\alpha = 10.0$) values of $\alpha$ can result in noticeable performance degradation, with the most significant drop observed at lower values of $\alpha$.

\begin{table}[]
\caption{Effect of reward range constraint ($\beta$) on reward reconstruction validity.}
\label{tab:beta_sensitivity}
\begin{tabular}{l|c|c}
\hline
        & Reward Error $\downarrow$ & Invalid $\tilde{r}$ Ratio(\%)$\downarrow$ \\
$\beta$ & Pong/Qbert   & Pong/Qbert                             \\ \hline
0.0     & 0.33/0.28    & 0.27/0.19                              \\
0.01    & 0.29/0.19    & 0.12/0.07                              \\
0.1     & 0.14/0.13    & 0.00/0.00                              \\
1.0     & 0.08/0.09    & 0.00/0.00                              \\
10.0    & 0.08/0.08    & 0.00/0.00                              \\ \hline
\end{tabular}
\end{table}

Similar to the sensitivity experiment on $\alpha$, we vary $\beta$ while keeping $\alpha = 1.0$ and $\gamma = 1.0$ fixed. 
Experimental results include the absolute reward reconstruction error and the proportion of invalid samples, where invalid samples are defined as reward values $\tilde{r}$ outside the legal range $[r_{\min}, r_{\max}]$ imposed by the environment.
As shown in Table~\ref{tab:beta_sensitivity}, varying $\beta$ within the range of 0.0 to 10.0 significantly alters the reward error and the proportion of invalid reconstruction samples.
In particular, a small value (e.g., $\beta=0.01$) effectively eliminates invalid samples. 
Moreover, experimental results show that, unlike state regularization, reward constraints are not highly sensitive to precise tuning.

To assess the contribution of the dynamics consistency regularization, we conduct a sensitivity analysis by varying the regularization weight $\gamma$.
Results (summarized in Appendix~\ref{supsec:para_sens}) demonstrate $\gamma$ around $1.0$ significantly improves transition realism and reduces sample dispersion, but overly large values over-constrain the reconstruction performance of RGIA.

$\quad$
\\
\noindent\textbf{Bias analysis of state prior $\mu_s$}.
To investigate the impact of deviations in the state prior $\mu_s$ used for regularization in RGIA on attack performance, we design an ablation experiment with biased priors.
Detailed experimental setup and results are provided in Appendix~\ref{supsec:state_prior}.
The experimental results demonstrate that higher-quality priors can significantly improve the accuracy of data reconstruction.
Specifically, in the Pong and Qbert environments, as the amount of data used to estimate the state prior increases, the reconstructed states become increasingly similar to the original ones, as measured by SSIM. 
Meanwhile, the action recovery accuracy (RA) also improves significantly.
However, when the prior data exceeds 2000 samples, neither SSIM nor RA increases significantly.
Similarly, Hopper and Walker2d show a similar trend. 
When the sample size exceeds 3000, $\mu_s$ no longer has a significant effect on the state MSE and RA.

$\quad$
\\
\noindent \textbf{Transition model analysis}.
In addition, we conduct an ablation study on the transition model $f$ trained with prior data (refer to Appendix~\ref{supsec:dynamic_model_affect}). 
Experimental results show that the transition model trained with only 0.3$\%$ of the data in the dataset as prior data can effectively regularize dynamic consistency, thereby improving the reconstruction ability of RGIA.
In addition, an interesting phenomenon is that as the amount of prior data increases, the transition model does not significantly improve the reconstruction ability of RGIA.

\section{Conclusions}
In this paper, we investigate gradient inversion attacks in FRL, which aim to recover local training data from publicly shared gradients.
To accurately reconstruct the training data, we propose a novel attack method called RGIA, which addresses the pseudo-solution problem by incorporating prior-knowledge-based regularization terms on states, rewards, and transition dynamics during the optimization process to ensure that the reconstructed data remains close to the true data transition distribution.
Extensive experimental results on control and autonomous driving tasks demonstrate that RGIA can effectively mitigate the pseudo-solution problem and ultimately recover the training data of the agent.
Future work will involve improving the RGIA to enable large-scale generation of training data and developing effective defense mechanisms against gradient privacy leakage in FRL.

\section*{Ethical Considerations}
Our work proposes a regularization-based gradient inversion attack (RGIA) for reconstructing training from policy gradients in FRL. While this research advances understanding of privacy vulnerabilities in RL systems, it also raises potential ethical concerns.

\noindent \textbf{Privacy and security risks.} The proposed method can reconstruct state, action, and reward sequences from shared gradients, which, in an FRL setting, may contain sensitive user or environment information. Malicious actors could misuse this technique to compromise the confidentiality of proprietary or personal datasets.

\noindent\textbf{Misuse scenarios.} If applied without safeguards, the method could facilitate targeted attacks on deployed RL agents, such as autonomous robots or industrial control systems, leading to the leakage of private data.

\noindent\textbf{Mitigation strategies.} To reduce these risks, we recommend limiting access to model parameters and gradients. Furthermore, as discussed in section~\ref{sec:defencse_dp}, introducing differential privacy into shared gradients can also mitigate RGIA attacks.

Overall, our work is intended to improve the security and robustness of FRL systems by revealing potential vulnerabilities and guiding the development of effective defenses.

\bibliographystyle{ACM-Reference-Format}
\bibliography{cas-refs}

@article{DBLP:abs-2103-06473,
  author       = {Aqeel Anwar and
                  Arijit Raychowdhury},
  title        = {Multi-Task Federated Reinforcement Learning with Adversaries},
  journal      = {CoRR},
  volume       = {abs/2103.06473},
  year         = {2021}
}

@inproceedings{DBLP:FangWG25,
  author       = {Minghong Fang and
                  Xilong Wang and
                  Neil Zhenqiang Gong},
  title        = {Provably Robust Federated Reinforcement Learning},
  booktitle    = {Proceedings of the {ACM} on Web Conference 2025, {WWW} 2025, Sydney,
                  NSW, Australia, 28 April 2025- 2 May 2025},
  pages        = {896--909},
  publisher    = {{ACM}},
  year         = {2025}
}

@article{DBLP:abs-2303-02725,
  author       = {Evelyn Ma and
                  Tiancheng Qin and
                  S. Rasoul Etesami},
  title        = {Local Environment Poisoning Attacks on Federated Reinforcement Learning},
  journal      = {CoRR},
  volume       = {abs/2303.02725},
  year         = {2023}
}

@inproceedings{DBLP:ZhuLH19,
  author       = {Ligeng Zhu and
                  Zhijian Liu and
                  Song Han},
  title        = {Deep Leakage from Gradients},
  booktitle    = {Advances in Neural Information Processing Systems 32: Annual Conference
                  on Neural Information Processing Systems 2019, NeurIPS 2019, December
                  8-14, 2019, Vancouver, BC, Canada},
  pages        = {14747--14756},
  year         = {2019}
}

@article{DBLP:ZhangXGYZ25,
  author       = {Xu Zhang and
                  Tao Xiang and
                  Shangwei Guo and
                  Fei Yang and
                  Tianwei Zhang},
  title        = {Deep Face Leakage: Inverting High-Quality Faces From Gradients Using
                  Residual Optimization},
  journal      = {{IEEE} Trans. Image Process.},
  volume       = {34},
  pages        = {1560--1572},
  year         = {2025}
}

@article{DBLP:abs-2001-02610,
  author       = {Bo Zhao and
                  Konda Reddy Mopuri and
                  Hakan Bilen},
  title        = {iDLG: Improved Deep Leakage from Gradients},
  journal      = {CoRR},
  volume       = {abs/2001.02610},
  year         = {2020}
}

@inproceedings{DBLP:HuangGSLA21,
  author       = {Yangsibo Huang and
                  Samyak Gupta and
                  Zhao Song and
                  Kai Li and
                  Sanjeev Arora},
  title        = {Evaluating Gradient Inversion Attacks and Defenses in Federated Learning},
  booktitle    = {Advances in Neural Information Processing Systems 34: Annual Conference
                  on Neural Information Processing Systems 2021, NeurIPS 2021, December
                  6-14, 2021, virtual},
  pages        = {7232--7241},
  year         = {2021}
}

@article{DBLP:KaissisZPRUTLMJ21,
  author       = {Georgios Kaissis and
                  Alexander Ziller and
                  Jonathan Passerat{-}Palmbach and
                  Th{\'{e}}o Ryffel and
                  Dmitrii Usynin and
                  Andrew Trask and
                  Ion{\'{e}}sio Lima and
                  Jason Mancuso and
                  Friederike Jungmann and
                  Marc{-}Matthias Steinborn and
                  Andreas Saleh and
                  Marcus R. Makowski and
                  Daniel Rueckert and
                  Rickmer Braren},
  title        = {End-to-end privacy preserving deep learning on multi-institutional
                  medical imaging},
  journal      = {Nat. Mach. Intell.},
  volume       = {3},
  number       = {6},
  pages        = {473--484},
  year         = {2021}
}

@inproceedings{DBLP:FangCWWX23,
  author       = {Hao Fang and
                  Bin Chen and
                  Xuan Wang and
                  Zhi Wang and
                  Shu{-}Tao Xia},
  title        = {{GIFD:} {A} Generative Gradient Inversion Method with Feature Domain
                  Optimization},
  booktitle    = {{IEEE/CVF} International Conference on Computer Vision, {ICCV} 2023,
                  Paris, France, October 1-6, 2023},
  pages        = {4944--4953},
  publisher    = {{IEEE}},
  year         = {2023}
}

@article{pearson1901liii,
  title={LIII. On lines and planes of closest fit to systems of points in space},
  author={Pearson, Karl},
  journal={The London, Edinburgh, and Dublin philosophical magazine and journal of science},
  volume={2},
  number={11},
  pages={559--572},
  year={1901},
  publisher={Taylor \& Francis}
}

@inproceedings{DBLP:Shi00Z00L25,
  author       = {Shanghao Shi and
                  Ning Wang and
                  Yang Xiao and
                  Chaoyu Zhang and
                  Yi Shi and
                  Y. Thomas Hou and
                  Wenjing Lou},
  title        = {Scale-MIA: {A} Scalable Model Inversion Attack against Secure Federated
                  Learning via Latent Space Reconstruction},
  booktitle    = {32nd Annual Network and Distributed System Security Symposium, {NDSS}
                  2025, San Diego, California, USA, February 24-28, 2025},
  publisher    = {The Internet Society},
  year         = {2025}
}

@article{DBLP:GoodfellowPMXWOCB14,
  author       = {Ian J. Goodfellow and
                  Jean Pouget{-}Abadie and
                  Mehdi Mirza and
                  Bing Xu and
                  David Warde{-}Farley and
                  Sherjil Ozair and
                  Aaron C. Courville and
                  Yoshua Bengio},
  title        = {Generative Adversarial Networks},
  journal      = {CoRR},
  volume       = {abs/1406.2661},
  year         = {2014}
}

@inproceedings{DBLP:KostrikovNL22,
  author       = {Ilya Kostrikov and
                  Ashvin Nair and
                  Sergey Levine},
  title        = {Offline Reinforcement Learning with Implicit Q-Learning},
  booktitle    = {The Tenth International Conference on Learning Representations, {ICLR}
                  2022, Virtual Event, April 25-29, 2022},
  publisher    = {OpenReview.net},
  year         = {2022}
}

@inproceedings{DBLP:UppuluriPMKC25,
  author       = {Bhargava Uppuluri and
                  Anjel Patel and
                  Neil Mehta and
                  Sridhar Kamath and
                  Pratyush Chakraborty},
  title        = {CuRLA: Curriculum Learning Based Deep Reinforcement Learning for Autonomous
                  Driving},
  booktitle    = {Proceedings of the 17th International Conference on Agents and Artificial
                  Intelligence, {ICAART} 2025 - Volume 3, Porto, Portugal, February
                  23-25, 2025},
  pages        = {435--442},
  publisher    = {{SCITEPRESS}},
  year         = {2025}
}

@article{DBLP:LiuYJLL24,
  author       = {Jia Liu and
                  Jianwen Yin and
                  Zhengmin Jiang and
                  Qingyi Liang and
                  Huiyun Li},
  title        = {Attention-Based Distributional Reinforcement Learning for Safe and
                  Efficient Autonomous Driving},
  journal      = {{IEEE} Robotics Autom. Lett.},
  volume       = {9},
  number       = {9},
  pages        = {7477--7484},
  year         = {2024}
}

@article{DBLP:MartinBMORB25,
  author       = {Adrian Martin and
                  Isabel de la Bandera and
                  Adriano Mendo and
                  Jos{\'{e}} Outes and
                  Juan Ramiro{-}Moreno and
                  Raquel Barco},
  title        = {Federated Deep Reinforcement Learning for {ENDC} Optimization},
  journal      = {{IEEE} Trans. Mob. Comput.},
  volume       = {24},
  number       = {6},
  pages        = {5525--5535},
  year         = {2025}
}

@article{DBLP:YouDLGWS25,
  author       = {Zhichao You and
                  Xuewen Dong and
                  Ximeng Liu and
                  Sheng Gao and
                  Yongzhi Wang and
                  Yulong Shen},
  title        = {Location Privacy Preservation Crowdsensing With Federated Reinforcement
                  Learning},
  journal      = {{IEEE} Trans. Dependable Secur. Comput.},
  volume       = {22},
  number       = {3},
  pages        = {1877--1894},
  year         = {2025}
}

@article{DBLP:NguyenB25,
  author       = {Van Tuan Nguyen and
                  Razvan Beuran},
  title        = {FedMSE: Semi-supervised federated learning approach for IoT network
                  intrusion detection},
  journal      = {Comput. Secur.},
  volume       = {151},
  pages        = {104337},
  year         = {2025}
}

@inproceedings{DBLP:0004QX24,
  author       = {Zixin Zhang and
                  Fan Qi and
                  Changsheng Xu},
  title        = {Enhancing Storage and Computational Efficiency in Federated Multimodal
                  Learning for Large-Scale Models},
  booktitle    = {Forty-first International Conference on Machine Learning, {ICML} 2024,
                  Vienna, Austria, July 21-27, 2024},
  publisher    = {OpenReview.net},
  year         = {2024}
}

@article{DBLP:LiangLZLZ23,
  author       = {Haotian Liang and
                  Youqi Li and
                  Chuan Zhang and
                  Ximeng Liu and
                  Liehuang Zhu},
  title        = {{EGIA:} An External Gradient Inversion Attack in Federated Learning},
  journal      = {{IEEE} Trans. Inf. Forensics Secur.},
  volume       = {18},
  pages        = {4984--4995},
  year         = {2023}
}

@article{DBLP:ChangZ24,
  author       = {Wenhan Chang and
                  Tianqing Zhu},
  title        = {Gradient-based defense methods for data leakage in vertical federated
                  learning},
  journal      = {Comput. Secur.},
  volume       = {139},
  pages        = {103744},
  year         = {2024}
}

@inproceedings{DBLP:MnihBMGLHSK16,
  author       = {Volodymyr Mnih and
                  Adri{\`{a}} Puigdom{\`{e}}nech Badia and
                  Mehdi Mirza and
                  Alex Graves and
                  Timothy P. Lillicrap and
                  Tim Harley and
                  David Silver and
                  Koray Kavukcuoglu},
  editor       = {Maria{-}Florina Balcan and
                  Kilian Q. Weinberger},
  title        = {Asynchronous Methods for Deep Reinforcement Learning},
  booktitle    = {Proceedings of the 33nd International Conference on Machine Learning,
                  {ICML} 2016, New York City, NY, USA, June 19-24, 2016},
  series       = {{JMLR} Workshop and Conference Proceedings},
  volume       = {48},
  pages        = {1928--1937},
  publisher    = {JMLR.org},
  year         = {2016}
}

@inproceedings{DBLP:FujimotoMP19,
  author       = {Scott Fujimoto and
                  David Meger and
                  Doina Precup},
  editor       = {Kamalika Chaudhuri and
                  Ruslan Salakhutdinov},
  title        = {Off-Policy Deep Reinforcement Learning without Exploration},
  booktitle    = {Proceedings of the 36th International Conference on Machine Learning,
                  {ICML} 2019, 9-15 June 2019, Long Beach, California, {USA}},
  series       = {Proceedings of Machine Learning Research},
  volume       = {97},
  pages        = {2052--2062},
  publisher    = {{PMLR}},
  year         = {2019}
}

@inproceedings{DBLP:YinMVAKM21,
  author       = {Hongxu Yin and
                  Arun Mallya and
                  Arash Vahdat and
                  Jos{\'{e}} M. {\'{A}}lvarez and
                  Jan Kautz and
                  Pavlo Molchanov},
  title        = {See Through Gradients: Image Batch Recovery via GradInversion},
  booktitle    = {{IEEE} Conference on Computer Vision and Pattern Recognition, {CVPR}
                  2021, virtual, June 19-25, 2021},
  pages        = {16337--16346},
  publisher    = {Computer Vision Foundation / {IEEE}},
  year         = {2021}
}

@article{kunt1985second,
  title={Second-generation image-coding techniques},
  author={Kunt, Murat and Ikonomopoulos, Athanassios and Kocher, Michel},
  journal={Proceedings of the IEEE},
  volume={73},
  number={4},
  pages={549--574},
  year={1985},
  publisher={IEEE}
}

@article{wang2004image,
  title={Image quality assessment: from error visibility to structural similarity},
  author={Wang, Zhou and Bovik, Alan C and Sheikh, Hamid R and Simoncelli, Eero P},
  journal={IEEE transactions on image processing},
  volume={13},
  number={4},
  pages={600--612},
  year={2004},
  publisher={IEEE}
}

@inproceedings{DBLP:KianiKDDB25,
  author       = {Shahrzad Kiani and
                  Nupur Kulkarni and
                  Adam Dziedzic and
                  Stark C. Draper and
                  Franziska Boenisch},
  title        = {Differentially Private Federated Learning with Time-Adaptive Privacy
                  Spending},
  booktitle    = {The Thirteenth International Conference on Learning Representations,
                  {ICLR} 2025, Singapore, April 24-28, 2025},
  publisher    = {OpenReview.net},
  year         = {2025}
}

@inproceedings{DBLP:LiuKC24,
  author       = {Wenye Liu and
                  Nazim Altar Koca and
                  Chip{-}Hong Chang},
  title        = {Efficient Fast Additive Homomorphic Encryption Cryptoprocessor for
                  Privacy-Preserving Federated Learning Aggregation},
  booktitle    = {Design, Automation {\&} Test in Europe Conference {\&} Exhibition,
                  {DATE} 2024, Valencia, Spain, March 25-27, 2024},
  pages        = {1--6},
  publisher    = {{IEEE}},
  year         = {2024}
}

@inproceedings{DBLP:OviDRG23,
  author       = {Pretom Roy Ovi and
                  Emon Dey and
                  Nirmalya Roy and
                  Aryya Gangopadhyay},
  title        = {Mixed Quantization Enabled Federated Learning to Tackle Gradient Inversion
                  Attacks},
  booktitle    = {{IEEE/CVF} Conference on Computer Vision and Pattern Recognition,
                  {CVPR} 2023 - Workshops, Vancouver, BC, Canada, June 17-24, 2023},
  pages        = {5046--5054},
  publisher    = {{IEEE}},
  year         = {2023}
}

\appendix

\clearpage

\section{Theorem Proof}
\label{supsec:proof}

\subsection{Lemma Proof}
\label{supsec:lemma_proof}
\begin{lemma}[Reward Space Constraint]
\label{lem:sup_reward_constraint}
Given the optimization problem:
$\min_{\tilde{x}} J(\tilde{x}) = F(\tilde{x}) + \lambda \mathcal{R}_r(\tilde{r}), \quad \lambda > 0$,
where $F(\tilde{x})$ is the primary objective function and $\mathcal{R}_r$ is the regularization term defined as:
$\mathcal{R}_r(\tilde{r}) = \left(\text{ReLU}(\tilde{r} - r_{\max})\right)^2 + \left(\text{ReLU}(r_{\min} - \tilde{r})\right)^2$,
then for any optimal solution $\tilde{x}^*$, the reward component $\tilde{r}^*$ must satisfy $\tilde{r}^* \in [r_{\min}, r_{\max}]$. This holds under the assumption that $F(\tilde{x})$ is continuous with respect to $\tilde{r}$.
\end{lemma}

\begin{proof}
We proceed by contradiction. Assume there exists an optimal solution $\tilde{x}^*$ such that $\tilde{r}^* \notin [r_{\min}, r_{\max}]$. There are two cases to consider:

\textbf{Case 1: $\tilde{r}^* > r_{\max}$} \\
Since $\tilde{r}^* > r_{\max} \geq r_{\min}$, the regularization term evaluates to:
\begin{equation}
    \mathcal{R}_r(\tilde{r}^*) = (\tilde{r}^* - r_{\max})^2 > 0.
\end{equation}
Consider a new candidate solution $\tilde{x}_{\text{new}}$ identical to $\tilde{x}^*$ except for the reward component, where we set:
\begin{equation}
    \tilde{r}_{\text{new}} = \frac{\tilde{r}^* + r_{\max}}{2}.
\end{equation}
Note that $\tilde{r}_{\text{new}} > r_{\max}$ since $\tilde{r}^* > r_{\max}$. The regularization term at this new point is:
\begin{equation}
    \mathcal{R}_r(\tilde{r}_{\text{new}}) = \left(\frac{\tilde{r}^* + r_{\max}}{2} - r_{\max}\right)^2 = \left(\frac{\tilde{r}^* - r_{\max}}{2}\right)^2.
\end{equation}
The difference in total objective values is:
\begin{equation}
    \begin{split}
        J(\tilde{x}_{\text{new}}) - J(\tilde{x}^*) &= \underbrace{\left[F(\tilde{x}_{\text{new}}) - F(\tilde{x}^*)\right]}_{\Delta F} \\ &+ \lambda \left[ \left(\frac{\tilde{r}^* - r_{\max}}{2}\right)^2 - (\tilde{r}^* - r_{\max})^2 \right].
    \end{split}
\end{equation}
Simplify the regularization difference:
\begin{equation}
    \lambda \left[ \frac{(\tilde{r}^* - r_{\max})^2}{4} - (\tilde{r}^* - r_{\max})^2 \right] = -\lambda \frac{3}{4} (\tilde{r}^* - r_{\max})^2 < 0.
\end{equation}
By continuity of $F$, for $\epsilon = \lambda \frac{3}{8} (\tilde{r}^* - r_{\max})^2 > 0$, there exists $\delta > 0$ such that:
\begin{equation}
    |\tilde{r}_{\text{new}} - \tilde{r}^*| < \delta \implies |\Delta F| < \epsilon.
\end{equation}
Select $\tilde{r}_{\text{new}}$ close enough to $\tilde{r}^*$ to satisfy this condition. Then:
\begin{equation}
    \begin{split}
        J(\tilde{x}_{\text{new}}) - J(\tilde{x}^*) &< \epsilon - \lambda \frac{3}{4} (\tilde{r}^* - r_{\max})^2 \\ &= \lambda \frac{3}{8} (\tilde{r}^* - r_{\max})^2 - \lambda \frac{3}{4} (\tilde{r}^* - r_{\max})^2 \\ &= -\lambda \frac{3}{8} (\tilde{r}^* - r_{\max})^2 < 0.
    \end{split}
\end{equation}

Thus $J(\tilde{x}_{\text{new}}) < J(\tilde{x}^*)$, contradicting the optimality of $\tilde{x}^*$.

\textbf{Case 2: $\tilde{r}^* < r_{\min}$} \\
The proof is symmetric to Case 1. Set:
\begin{equation}
    \tilde{r}_{\text{new}} = \frac{\tilde{r}^* + r_{\min}}{2} < r_{\min}.
\end{equation}
The regularization difference is:
\begin{equation}
    \lambda \left[ \left(\frac{\tilde{r}^* - r_{\min}}{2}\right)^2 - (\tilde{r}^* - r_{\min})^2 \right] = -\lambda \frac{3}{4} (\tilde{r}^* - r_{\min})^2 < 0.
\end{equation}

By continuity of $F$, we can find $\tilde{r}_{\text{new}}$ sufficiently close to $\tilde{r}^*$ such that:
\[
J(\tilde{x}_{\text{new}}) - J(\tilde{x}^*) < 0,
\]
again contradicting optimality.

Since both cases lead to contradictions, our initial assumption must be false. Therefore, for any optimal solution $\tilde{x}^*$, we must have $\tilde{r}^* \in [r_{\min}, r_{\max}]$.
\end{proof}



Now, we combine the above remarks and lemma to give the following core theorem.

\subsection{Theorem Proof}
\label{supsec:theorem_proof}
\begin{theorem}[Regularization for Compressing the Solution Space]
Let $\mathcal{X}^*_{\text{reg}}$ denote the set of global optima to the regularized problem. By introducing a prior-informed regularization term $\mathcal{R}_{\text{total}}(x)$, the solution space is effectively compressed such that 
$\mathcal{X}^*_{\text{reg}} \subseteq \mathcal{X}^\epsilon_{\text{orig}} \cap \mathcal{C}_{\text{prior}}$,
where $\mathcal{X}^\epsilon_{\text{orig}}$ denotes the $\epsilon$-approximate solution space with respect to gradient matching, and $\mathcal{C}_{\text{prior}}$ is the constraint set induced by the regularization priors.
\end{theorem}

\begin{proof}
Consider the regularized optimization problem:
\begin{equation}
\min_{\tilde{x}} J(\tilde{x}) := \left\| \nabla_\theta \mathcal{L}(\tilde{x}; \theta) - g_{\text{obs}} \right\|^2 
+ \lambda \left( \alpha \mathcal{R}_s(\tilde{s}) + \beta \mathcal{R}_r(\tilde{r}) + \gamma \mathcal{R}_{f}(\tilde{x}) \right).
\tag{P2}
\end{equation}

Let $\tilde{x}^* = (\tilde{s}^*, \tilde{a}^*, \tilde{r}^*, \tilde{s}'^*)$ be any global minimizer of $J(\tilde{x})$. Then for any $\tilde{x}$,
\begin{equation}
    J(\tilde{x}^*) \leq J(\tilde{x}).
\end{equation}

In particular, consider the true data sample $x_{\text{real}} = (s_{\text{real}}, a_{\text{real}}, r_{\text{real}}, s'_{\text{real}})$ such that 
\begin{equation}
    g_{\text{obs}} = \nabla_\theta \mathcal{L}(x_{\text{real}}; \theta).
\end{equation}
Since $x_{\text{real}}$ is physically valid, it satisfies:
\begin{equation}
    \begin{split}
        \mathcal{R}_s(s_{\text{real}}) &= \|s_{\text{real}} - \mu_s\|^2 = \delta_s \geq 0, \\
\mathcal{R}_r(r_{\text{real}}) &= 0, \\
\mathcal{R}_{f}(x_{\text{real}}) &= \|f(s_{\text{real}}, a_{\text{real}}) - s'_{\text{real}}\|^2 = \delta_{\text{dyn}} \geq 0.
    \end{split}
\end{equation}

Then the objective at $x_{\text{real}}$ is:
\begin{equation}
    J(x_{\text{real}}) = \underbrace{\left\| \nabla_\theta \mathcal{L}(x_{\text{real}}; \theta) - g_{\text{obs}} \right\|^2}_{= 0}
+ \lambda \left( \alpha \delta_s + \gamma \delta_{\text{dyn}} \right) \triangleq \epsilon_J.
\end{equation}

Hence, since $\tilde{x}^*$ minimizes $J$, we have:
\begin{equation}
\left\| \nabla_\theta \mathcal{L}(\tilde{x}^*; \theta) - g_{\text{obs}} \right\|^2 + \lambda \mathcal{R}_{\text{total}}(\tilde{x}^*) \leq \epsilon_J. 
\end{equation}

From this, we can extract two implications:

\paragraph{(1) Gradient Matching Condition}
\begin{equation}
\left\| \nabla_\theta \mathcal{L}(\tilde{x}^*; \theta) - g_{\text{obs}} \right\|^2 \leq \epsilon_J. \tag{2}
\end{equation}
This implies $\tilde{x}^* \in \mathcal{X}^\epsilon_{\text{orig}}$.

\paragraph{(2) Regularization Constraint}
From (1), we also have:
\begin{equation}
\lambda \mathcal{R}_{\text{total}}(\tilde{x}^*) \leq \epsilon_J 
\quad \Longrightarrow \quad 
\mathcal{R}_{\text{total}}(\tilde{x}^*) \leq \frac{\epsilon_J}{\lambda}.
\end{equation}

Expanding the regularizer gives:
\begin{equation}
    \alpha \mathcal{R}_s(\tilde{s}^*) + \beta \mathcal{R}_r(\tilde{r}^*) + \gamma \mathcal{R}_{f}(\tilde{x}^*) \leq \frac{\epsilon_J}{\lambda}.
\end{equation}

Now, we analyze each term separately.

\paragraph{(a) State Prior:} If $\alpha > 0$,
\begin{equation}
    \mathcal{R}_s(\tilde{s}^*) = \|\tilde{s}^* - \mu_s\|^2 \leq \frac{\epsilon_J}{\lambda \alpha}. 
\end{equation}
Thus, $\tilde{s}^*$ lies in a hypersphere centered at $\mu_s$.

\paragraph{(b) Reward Prior:} If $\beta > 0$,
\begin{equation}
    \mathcal{R}_r(\tilde{r}^*) = \left( \text{ReLU}(\tilde{r}^* - r_{\max}) \right)^2 + \left( \text{ReLU}(r_{\min} - \tilde{r}^*) \right)^2 \leq \frac{\epsilon_J}{\lambda \beta}.
\end{equation}
This implies:
\begin{equation}
    \tilde{r}^* \in [r_{\min} - \epsilon, r_{\max} + \epsilon], \quad \text{for small } \epsilon.
\end{equation}

\paragraph{(c) Dynamics Prior:} If $\gamma > 0$,
\begin{equation}
    \mathcal{R}_{f}(\tilde{x}^*) = \|f(\tilde{s}^*, \tilde{a}^*) - \tilde{s}'^*\|^2 \leq \frac{\epsilon_J}{\lambda \gamma}.
\end{equation}

This implies \((\tilde{s}^*, \tilde{a}^*, \tilde{s}'^*)\) lies close to the manifold defined by the transition model.

\paragraph{Conclusion:} From inequalities (2), (5), (7), and (8), we conclude:
\begin{equation}
    \tilde{x}^* \in \mathcal{X}^\epsilon_{\text{orig}} \cap \mathcal{C}_{\text{prior}},
\end{equation}
which completes the proof.
\end{proof}

\section{Additional environmental setup information}

\subsection{Environmental Information}
\label{supsec:envion_infor}
The Atari environment serves as a classic benchmark for both offline and online RL, widely used to evaluate agent performance based on pixel input.
The following provides brief descriptions of four widely studied Atari games used in RL benchmarks:
\begin{itemize}
    \item  Pong is a table tennis match game where the objective is to score points and defeat opponents in a left-right duel. In this environment, the agent must learn to precisely rebound the ball using its paddle while simultaneously anticipating the opponent's behavior.

    \item Breakout is a brick-breaking game where the player controls a paddle to rebound a ball and break the bricks positioned above. To achieve this efficiently, the player must master the relationship between the rebound angle and the speed of the ball to maximize breaking efficiency.

    \item Qbert is a jumping game played on an isometric grid, where the goal is to make the character jump onto each cell to change its color. In this setting, the agent needs to handle complex spatial navigation, jump control, and avoid collisions with enemies.

    \item Seaquest is a side‑scrolling underwater shooting game where an agent controls a submarine to repel enemies and rescue divers, while periodically surfacing to replenish oxygen. These gameplay dynamics require the agent to coordinate multiple strategic goals, including offense, defense, navigation, and resource management.

\end{itemize}

MuJoCo is a high-performance physics engine widely used for simulating continuous control tasks in RL, particularly well-suited for detailed modeling of robot dynamics and contact mechanics.
Below are the four common control tasks in MuJoCo:

\begin{itemize}
    \item Hopper is a single-legged hopping robot designed to maintain stable forward locomotion. The challenge of the Hopper lies in maintaining body balance and absorbing impact during single–leg landings. Hence, the agent must learn to coordinate leg joints of  to achieve continuous hopping without falling.

    \item Walker2d simulates a two-dimensional bipedal robot designed to walk forward, which requires the agent to maintain dynamic balance and coordinate multiple joints. Consequently, the policy must learn to achieve an efficient gait while preventing falls.

    \item Ant is a quadruped robot with 8 controllable joints, designed to crawl forward on a two-dimensional plane. The crawling task presents a high-dimensional action space, where the primary challenge is to coordinate the movements of multiple legs to achieve stable and fast locomotion.

    \item Halfcheetah is a two-dimensional robotic agent that represents a simplified, bipedal version of a cheetah, equipped with multiple joints in its hind legs and torso to simulate agile movement. 
    This design enables dynamic locomotion patterns similar to running. 
    In this environment, the goal of the agent is to move forward efficiently by learning to coordinate the movements of its hind legs and torso.
    
\end{itemize}

Car Racing is a continuous control task with pixel-based observations and is commonly used to evaluate the performance of RL algorithms under high-dimensional visual inputs and continuous action spaces.
In this environment, the agent controls a car that navigates a randomly generated track, aiming to cover as much of the track as possible within a limited time to maximize its score.
Each state is represented by an RGB image capturing the road layout and the position of the car from a top-down perspective. 
The agent operates in a continuous action space defined by three control variables: steering angle, throttle, and brake.

FrozenLake is a classic discrete reinforcement learning environment where agents need to move from the starting point to the target position on a frozen lake surface while avoiding falling into an ice cave.
The environment is composed of a grid, where each cell can be a start point (S), safe ice (F), hole (H), or goal (G).
In this environment, the agent can take four actions: move up, down, left, or right.
However, due to slippery ground conditions, actions may stochastically deviate from the intended direction, which introduces stochasticity into policy learning.


\subsection{Evaluation Metric Describe}
\label{supsec:metirc}
For gradient inversion attacks involving image inputs, we primarily assess the similarity between the reconstructed and original images using PSNR and SSIM.
In particular, PSNR quantifies image distortion based on the pixel-wise mean squared error (MSE), with higher values indicating better reconstruction quality and lower distortion.
Unlike PSNR, SSIM focuses on measuring structural similarity in images and aligns with the perception of the human visual system.
Specifically, SSIM evaluates the similarity between two images by comparing their luminance, contrast, and structure components.
SSIM ranges from -1 to 1, with values closer to 1 signifying greater structural similarity between two images.
For low-dimensional state inputs, we employ MSE to directly measure the mean squared difference between the features of the original and reconstructed states.
For actions, we use recovery accuracy (RA) to measure the reconstruction capability of RGIA, defined as $\frac{n}{m} \times 100\%$, where $n$ is the number of correctly reconstructed actions and $m$ is the total number of evaluated actions.
Lastly, since rewards are continuous variables, we also use MSE to evaluate the accuracy of reconstructed rewards.

In addition, to better evaluate the impact of regularization terms on the reconstructed samples, we introduce consistency metrics, which include Euclidean distance (ED), Silhouette Score (SS), Covariance Determinant (CD), and Transition Error (TE).
Specifically, ED calculates the average pairwise Euclidean distance between reconstructed states, where a lower value indicates greater similarity.
SS quantifies structural consistency by measuring the clustering quality of reconstructed samples under multiple random initializations.
CD is the determinant of the state covariance matrix, which measures the spread of the reconstruction distribution, with lower values indicating tighter, more consistent solutions.
TE is an MSE between $f(\tilde{s},\tilde{a})$ and $\tilde{s}'$, which indicates whether the reconstructed samples are consistent with the environment dynamics.

\subsection{Baselines}

\begin{itemize}
    \item DGL iteratively updates virtual inputs and labels through an optimization algorithm, so that the gradients generated by the virtual data are close to the real gradients, thereby accurately restoring the original training data.

    \item GradInversion transforms gradient inversion into a joint optimization problem with strong prior constraints, which uses gradient matching as the main loss and BN statistics-total variation-group consistency as regularization to drive multiple randomly initialized noisy images to converge to the high-fidelity original image.

    \item GIFD performs layer-wise optimization of intermediate features in the generator under an L1-ball constraint, which significantly enlarges the search space while suppressing distortion. This strategy enables the attack to achieve pixel-level reconstruction even under distribution shift and noise.

    \item DFLeak progressively integrates high-frequency details from Prior-Free Face Restoration (PFFR) into the reconstructed image in a residual manner during gradient matching iterations. In addition, it introduces a pixel-wise update scheduling strategy, which applies a decay factor to the gradients in the fusion regions to suppress the smoothing side effects of regularization terms, thereby continuously preserving facial textures.
\end{itemize}

\section{Additional experiments}

\subsection{Comparative Experiment of High-dimensional Environment}
\label{supsec:high-dim}
To intuitively evaluate the effectiveness of RGIA, we conduct gradient inversion attacks in RFL based on the AC framework. 
In this RFL-AC setting, each local agent interacts with the environment and updates its local parameters independently, without sharing raw trajectories. During the policy optimization stage, the agent uploads gradients of the value function to the central server for aggregation.
RGIA intercepts these uploaded gradients and aims to reconstruct the private training data, including states, rewards, and actions. 
To quantitatively evaluate the reconstruction performance, we compare RGIA with four representative baselines: DGL, GradInv, GIFD, and DFLeak. 

\subsection{Defense Experiments}
\label{supsec:dp_def}
To evaluate the robustness of RGIA against differential privacy (DP)-based defense mechanisms, we simulate a privacy-preserving training scenario by adding random noise to the gradients during policy training. 
Specifically, we consider two commonly used DP noise distributions: Gaussian noise and Laplace noise. 
The noise is injected into the gradients of the value function network in the AC architecture.
The configuration of each noise is as follows:
\begin{itemize}
    \item Distributions: (1) Gaussian noise: $\mathcal{N}(0, \sigma^2)$ and (2) Laplace noise: $\text{Lap}(0, b)$, where $b = \sqrt{2}\sigma$ to match variance with Gaussian noise.

    \item Mean: 0

    \item Variance ($\sigma^2$) range: $[10^{-1},\ 10^{-2},\ 10^{-3},\ 10^{-4},\ 10^{-5}]$
\end{itemize}
This range allows us to evaluate the impact of different privacy levels, from strong (high noise) to weak (low noise) protection.

\begin{table}[]
\centering
\caption{Laplace noise defense results.}
\begin{tabular}{llcccc}
\hline
Variance & Environment                                                             & Loss$\downarrow$                                                                & Rewards$\uparrow$                                                     & MSE$\downarrow$                                                                  & RA$\uparrow$                                                         \\ \hline
1e-1     & \begin{tabular}[c]{@{}l@{}}Hopper\\ Walker2d\\ Halfcheetah\end{tabular} & \begin{tabular}[c]{@{}c@{}}0.81\\ 0.48\\ 0.91\end{tabular}          & \begin{tabular}[c]{@{}c@{}}21.4\\ 30.3\\ 20.7\end{tabular}  & \begin{tabular}[c]{@{}c@{}}0.31\\ 0.27\\ 0.38\end{tabular}          & \begin{tabular}[c]{@{}c@{}}0.03\\ 0.08\\ 0.09\end{tabular} \\ \hline
1e-2     & \begin{tabular}[c]{@{}l@{}}Hopper\\ Walker2d\\ Halfcheetah\end{tabular} & \begin{tabular}[c]{@{}c@{}}0.56\\ 0.30\\ 0.67\end{tabular}          & \begin{tabular}[c]{@{}c@{}}34.4\\ 40.8\\ 42.5\end{tabular}  & \begin{tabular}[c]{@{}c@{}}0.23\\ 0.20\\ 0.27\end{tabular}          & \begin{tabular}[c]{@{}c@{}}0.13\\ 0.14\\ 0.12\end{tabular} \\ \hline
1e-3     & \begin{tabular}[c]{@{}l@{}}Hopper\\ Walker2d\\ Halfcheetah\end{tabular} & \begin{tabular}[c]{@{}c@{}}0.025\\ 0.027\\ 0.034\end{tabular}       & \begin{tabular}[c]{@{}c@{}}70.3\\ 87.6\\ 65.9\end{tabular}  & \begin{tabular}[c]{@{}c@{}}0.17e-2\\ 0.17e-1\\ 0.18e-1\end{tabular} & \begin{tabular}[c]{@{}c@{}}0.62\\ 0.70\\ 0.68\end{tabular} \\ \hline
1e-4     & \begin{tabular}[c]{@{}l@{}}Hopper\\ Walker2d\\ Halfcheetah\end{tabular} & \begin{tabular}[c]{@{}c@{}}0.24e-2\\ 0.84e-2\\ 0.73e-2\end{tabular} & \begin{tabular}[c]{@{}c@{}}89.7\\ 102.3\\ 86.7\end{tabular} & \begin{tabular}[c]{@{}c@{}}0.11e-3\\ 0.14e-3\\ 0.17e-3\end{tabular} & \begin{tabular}[c]{@{}c@{}}0.79\\ 0.82\\ 0.74\end{tabular} \\ \hline
1e-5     & \begin{tabular}[c]{@{}l@{}}Hopper\\ Walker2d\\ Halfcheetah\end{tabular} & \begin{tabular}[c]{@{}c@{}}0.09-2\\ 0.11e-2\\ 0.11e-2\end{tabular}  & \begin{tabular}[c]{@{}c@{}}94.5\\ 106.9\\ 89.6\end{tabular} & \begin{tabular}[c]{@{}c@{}}0.09e-5\\ 0.13e-5\\ 0.18e-3\end{tabular} & \begin{tabular}[c]{@{}c@{}}0.83\\ 0.83\\ 0.82\end{tabular} \\ \hline
\end{tabular}
\label{tab:lap_nosie}
\end{table}

Table~\ref{tab:lap_nosie} presents the performance of the RGIA attack under Laplace noise defense across three environments (Hopper, Walker2d, and Halfcheetah) with varying noise variances from $1\text{e}^{-1}$ to $1\text{e}^{-5}$. We observe a consistent trend across all metrics, indicating that increasing the strength of differential privacy (i.e., higher noise variance) significantly reduces the attack effectiveness, while weakening the performance of the learned policy.

Furthermore, we investigate the impact of homomorphic encryption (HE) and gradient quantization (GQ) on the performance of RGIA.
To simplify experimental complexity, we implement a lightweight HE scheme that performs gradient encryption via additive operations before uploading the encrypted gradients. Concurrently, we apply GQ to map gradients into low-bit representations (e.g., 8-bit or 4-bit). 
Since these techniques obscure gradient distributions, we propose the gradient matching error (GME) metric to intuitively quantify gradient variations. The combined effects of HE and GQ on RGIA are summarized in Table~\ref{tab:sup_def_vary}.

\begin{table}[]
\caption{The impact of different defense methods on RGIA performance.}
\begin{tabular}{llcccc}
\hline
Environment & Defense method                                                       & GME$\downarrow$                                                        & Loss$\downarrow$                                                    & MSE$\downarrow$                                                        & Rewards$\uparrow$                                                    \\ \hline
Pong        & \begin{tabular}[c]{@{}l@{}}HE\\ GQ (4 bit)\\ GQ (8 bit)\end{tabular} & \begin{tabular}[c]{@{}c@{}}0.89\\ 0.51\\ 0.11\end{tabular}  & \begin{tabular}[c]{@{}c@{}}2.34\\ 1.35\\ 0.27\end{tabular} & \begin{tabular}[c]{@{}c@{}}0.56\\ 0.27\\ 0.071\end{tabular} & \begin{tabular}[c]{@{}c@{}}20.3\\ -17.4\\ 4.3\end{tabular} \\ \hline
Hopper      & \begin{tabular}[c]{@{}l@{}}HE\\ GQ (4 bit)\\ GQ (8 bit)\end{tabular} & \begin{tabular}[c]{@{}c@{}}0.63\\ 0.34\\ 0.09\end{tabular} & \begin{tabular}[c]{@{}c@{}}1.78\\ 0.73\\ 0.25\end{tabular} & \begin{tabular}[c]{@{}c@{}}0.79\\ 0.21\\ 0.038\end{tabular} & \begin{tabular}[c]{@{}c@{}}93.7\\ 34.6\\ 70.3\end{tabular} \\ \hline
\end{tabular}
\label{tab:sup_def_vary}
\end{table}

From the experimental results, homomorphic encryption can effectively defend against data-reconstruction attacks launched by RGIA without compromising the original algorithmic performance (i.e., reward).
However, it is well known that homomorphic encryption introduces substantial computational overhead, which significantly limits its practical application in FRL.
Although gradient quantization can prevent data reconstruction, it adversely affects the original performance of the FRL algorithm.
In particular, the performance is notably degraded when gradient values are converted to low precision (4 bit).

\subsection{Visualization Experiment}
We visualize reconstructed images from five methods (i.e., RGIA, DGL, GrandInv, GIFD, and DFLeak) across four environments: Pong, Breakout, Qbert, and Car Racing.
As shown in Figure~\ref{fig:sup_recon_image}, GrandInv, GIFD, and DFLeak introduce improved feature loss functions, enabling the reconstruction of samples that are more similar to the original images. 
However, since these methods cannot address the pseudo-solution problem in FRL, they occasionally reconstruct failed samples.
In particular, although DFLeak is capable of reconstructing high-quality states in most environments, the reconstructed states in the Qbert environment contain some noise (as indicated by the red ellipse).
Moreover, GrandInv reconstructed high-quality images in the Car Racing environment, but the reconstructed state is not the original state (as indicated by the red box).
\begin{figure}
    \centering
    \includegraphics[width=1.0\linewidth]{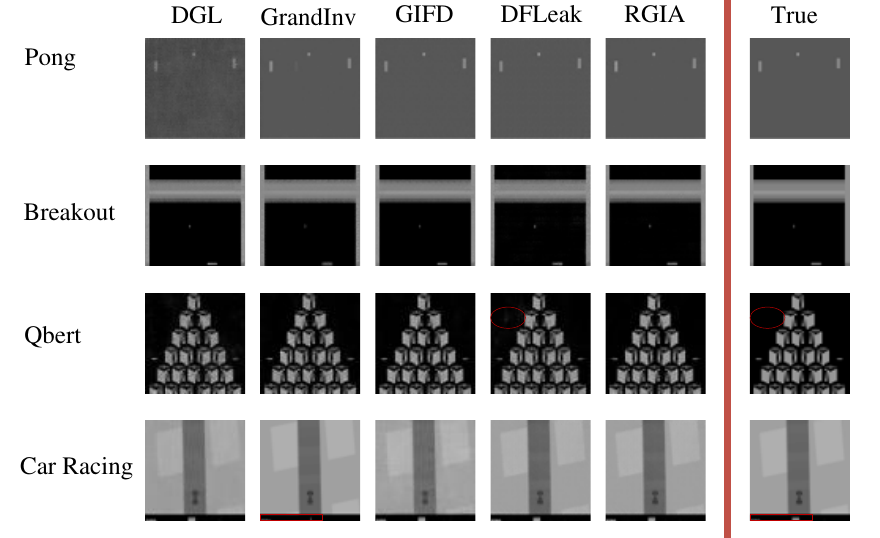}
    \caption{Reconstruction examples of different methods.}
    \label{fig:sup_recon_image}
\end{figure}

\subsection{Low-dimensional Reconstruction Experiment.}
\label{supsec:low_dim}
Since the state in FrozenLake is a discrete variable, we do not directly optimize the state variables but instead handle these variables similarly to discrete actions.
Specifically, we first generate a random vector of shape $[N, M, C]$, where $N$ is the batch size, $M$ is the state dimension, and $C$ is the number of grid value intervals.
Next, we optimize these variables using gradient descent and obtain the reconstructed state by determining the state values via index $C$.
The experimental results are shown in Figure~\ref{fig:discrete_state}, where each result is the average result of 10 random seeds.
As can be seen, all methods achieve remarkably high performance in the FrozenLake environment. 
This is primarily because FrozenLake is a relatively simple and deterministic environment, where the training states and actions can be precisely manipulated without the need for additional constraints.

\begin{figure}[]
    \centering
    \begin{subfigure}[b]{0.33\textwidth}
        \centering
        \includegraphics[width=\textwidth]{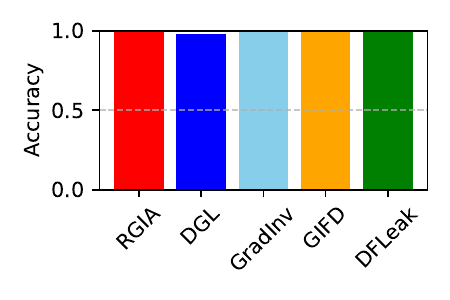}
    \end{subfigure}
    \begin{subfigure}[b]{0.33\textwidth}
        \centering
        \includegraphics[width=\textwidth]{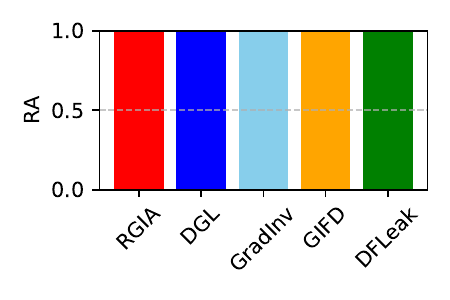}
    \end{subfigure}
    \caption{The comparison results of FrozenLake environments.}
    \label{fig:discrete_state}
\end{figure}

\subsection{Parameter Sensitivity Analysis}
\label{supsec:para_sens}

\begin{table}[]
\caption{Effect of reward range constraint ($\beta$) on reward reconstruction validity.}
\label{tab:beta_sensitivity}
\begin{tabular}{l|c|c}
\hline
        & Reward Error $\downarrow$ & Invalid $\tilde{r}$ Ratio(\%)$\downarrow$ \\
$\beta$ & Pong/Qbert   & Pong/Qbert                             \\ \hline
0.0     & 0.33/0.28    & 0.27/0.19                              \\
0.01    & 0.29/0.19    & 0.12/0.07                              \\
0.1     & 0.14/0.13    & 0.00/0.00                              \\
1.0     & 0.08/0.09    & 0.00/0.00                              \\
10.0    & 0.08/0.08    & 0.00/0.00                              \\ \hline
\end{tabular}
\end{table}

Similar to the sensitivity experiment on $\alpha$, we vary $\beta$ while keeping $\alpha = 1.0$ and $\gamma = 1.0$ fixed. 
Experimental results include the absolute reward reconstruction error and the proportion of invalid samples, where invalid samples are defined as reward values $\tilde{r}$ outside the legal range $[r_{\min}, r_{\max}]$ imposed by the environment.
As shown in Table~\ref{tab:beta_sensitivity}, varying $\beta$ within the range of 0.0 to 10.0 significantly alters the reward error and the proportion of invalid reconstruction samples.
In particular, a small value (e.g., $\beta=0.01$) effectively eliminates invalid samples. 
Moreover, experimental results show that, unlike state regularization, reward constraints are not highly sensitive to precise tuning.

To assess the contribution of the dynamics consistency regularization, we conduct a sensitivity analysis by varying the regularization weight $\gamma \in \{0.0, 0.01, 0.1, 1.0, 10.0\}$, while keeping the other regularization terms fixed ($\alpha = 1.0, \beta = 1.0$). 
This setup allows us to evaluate the effect of the transition prior that penalizes inconsistencies between the reconstructed next state $\tilde{s}'$ and the prediction of a pre-trained forward model $f(\tilde{s}, \tilde{a})$.
Here, we use three metrics: Transition Error (TS), Silhouette Score (SS), and Covariance Determinant (CD) to comprehensively evaluate the impact of different $\gamma$ values.

\begin{table}[]
\centering
\caption{Effect of dynamics consistency regularization ($\gamma$) on transition realism and sample concentration. The TE values reported in the table are the actual values scaled by a factor of $10^{-4}$.}
\label{tab:gamma_sensitivity}
\begin{tabular}{lccc}
\hline
         & TE$\downarrow$                 & SS$\uparrow$                 & CD$\downarrow$                     \\
$\gamma$ & Pong/Qbert         & Pong/Qbert         & Pong/Qbert             \\ \hline
0.0      & 0.31/0.35          & 0.28/0.19          & 5.3e-2/4.9e-2          \\
0.01     & 0.21/0.27          & 0.42/0.39          & 3.7e-2/4.1e-2          \\
0.1      & 0.12/0.18          & 0.66/0.59          & 1.5e-2/2.1e-2          \\
1.0      & \textbf{0.10/0.13} & \textbf{0.72/0.68} & \textbf{1.2e-3/1.3e-3} \\
10.0     & 0.11/0.13          & 0.69/0.66          & 2.1e-3/1.9e-3          \\ \hline
\end{tabular}
\end{table}

Table~\ref{tab:gamma_sensitivity} presents experimental results for varying $\gamma$ values in RGIA reconstruction.
When no dynamics prior is applied ($\gamma = 0.0$), the reconstructed samples exhibit larger TE, lower SS, and higher CD compared to those with constraints. 
These results indicate that many generated trajectories violate environment dynamics and are highly dispersed in the state space.
As $\gamma$ increases, three metrics improve significantly. 
For instance, at $\gamma = 1.0$, the TE decreases by over 67$\%$, and the SS increases to 0.72. 
Additionally, CD drops by nearly two orders of magnitude, demonstrating a pronounced shrinkage of the solution space.
Interestingly, further increasing $\gamma$ to 10.0 results in degraded metrics (i.e., increased TE and CD, and decreased SS).
This phenomenon suggests that overly aggressive regularization may over-constrain the optimization process, which reduces flexibility and diversity.

\subsection{State $\mu_s$}
\label{supsec:state_prior}
To investigate the impact of deviations in the state prior $\mu_s$ used for regularization in RGIA on attack performance, we design an ablation experiment with biased priors.
Specifically, the experiment involves a state-free prior and mean priors derived from datasets of varying sizes.
For the latter, we select 500, 1000, 2000, 3000, 4000 and all training samples to calculate the mean prior.
Ablation experiments are conducted in the Pong, Qbert, Hopper, and Walker2d environments, and the evaluation is performed using the SSIM, MSE, and RA metrics to quantify performance impact.
The experimental results are shown in Figure~\ref{fig:state_mu}, where $\mu_0$ denotes the state regularization not applied, and $u_5$, $u_{10}$, $u_{20}$, $u_{30}$, $u_{40}$, $u_{T}$ represent the state priors calculated with different data volumes, respectively.

\begin{figure}[]
    \centering
    \begin{subfigure}[b]{0.22\textwidth}
        \centering
        \includegraphics[width=\textwidth]{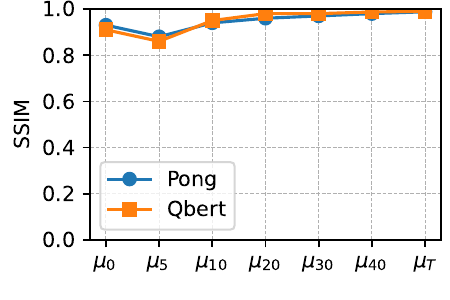}
    \end{subfigure}
    \begin{subfigure}[b]{0.22\textwidth}
        \centering
        \includegraphics[width=\textwidth]{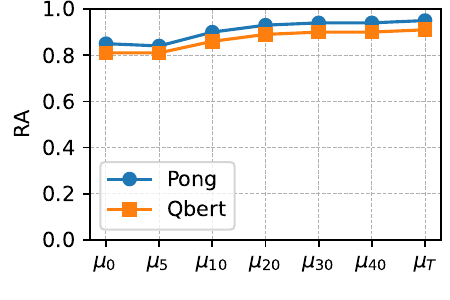}
    \end{subfigure}
    \begin{subfigure}[b]{0.22\textwidth}
        \centering
        \includegraphics[width=\textwidth]{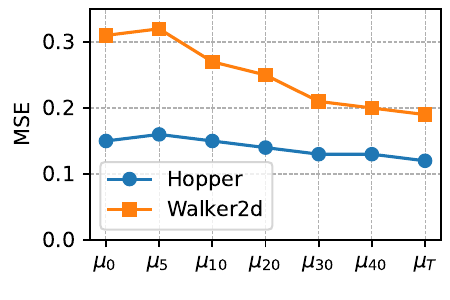}
    \end{subfigure}
    \begin{subfigure}[b]{0.21\textwidth}
        \centering
        \includegraphics[width=\textwidth]{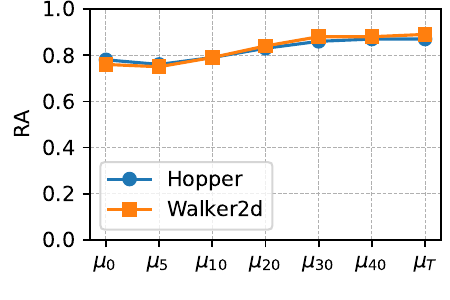}
    \end{subfigure}
    \caption{The impact of state prior $\mu_s$ on the performance of RGIA.}
    \label{fig:state_mu}
\end{figure}

Experimental results demonstrate that higher-quality priors can significantly improve the accuracy of data reconstruction.
Specifically, in the Pong and Qbert environments, as the amount of data used to estimate the state prior increases, the reconstructed states become increasingly similar to the original ones, as measured by SSIM. 
Meanwhile, the action recovery accuracy (RA) also improves significantly.
However, when the prior data exceeds 2000 samples, neither SSIM nor RA increases significantly.
This plateau may be explained by the fact that 2000 samples are sufficient to accurately estimate the state prior distribution $\mu_s$.
An interesting exception occurs when only 500 samples are used to calculate $\mu_s$: in this case, the reconstruction ability of RGIA is actually worse than that without state regularization. 
This decline likely results from a significant deviation of $\mu_s$ from the true distribution, which causes the optimization process to converge to a semantically incorrect solution.
Similarly, Hopper and Walker2d show a similar trend. 
When the sample size exceeds 3000, $\mu_s$ no longer has a significant effect on the state MSE and RA.

In summary, our ablation study demonstrates that while an accurate state prior $\mu_s$ is crucial for RGIA to achieve optimal attack performance, only a moderate amount of data (typically ranging from 2000 to 3000 samples across environments, representing a mere 0.3$\%$ of the 1M total training dataset) is sufficient to estimate a prior that saturates the reconstruction accuracy improvements.

\subsection{Effect of transition model Training Data Size on RGIA Performance.}
\label{supsec:dynamic_model_affect}
To evaluate the impact of transition model accuracy on the proposed RGIA method, we conduct an ablation study by varying the amount of data used to train the transition model $f$ in the dynamic consistency regularizer $\mathcal{R}_{f} = \| f(s,a) - s' \|^2$.
For each environment (Walker2d and Breakout), we train $f$ using different amounts of prior data $\{500, 1000, 2000, 3000, 4000, all\}$, where $all$ denotes that the transition model is trained using the entire dataset.
The model architecture and optimization settings are kept identical across all data settings, and the remaining RGIA hyperparameters are fixed to their default values.
After training $f$, we run RGIA with the learned transition model and evaluate reconstruction quality on 1,000 randomly selected target gradients.
For each data setting, the experiment is repeated with 10 different random seeds, and we report the mean values for three metrics: state MSE, action recovery accuracy (RA), and state transition error (TE).

\begin{table}[]
\caption{The influence of transition models on RGIA reconstruction capabilities. $0$ means that the dynamic consistency regularization is not used, and 1M means that the transition model is trained by all training data.}
\begin{tabular}{c|ccc|ccc}
\hline
\multirow{2}{*}{Numbers} & \multicolumn{3}{c|}{Walker2d}                   & \multicolumn{3}{c}{Breakout}                    \\
                         & MSE$\downarrow$ & RA$\uparrow$ & TE$\downarrow$ & MSE$\downarrow$ & RA$\uparrow$ & TE$\downarrow$ \\ \hline
0                        & 0.29e-7         & 0.77         & 0.16e-5        & 0.83e-6         & 0.82         & 0.23e-4        \\
500                      & 0.11e-6         & 0.77         & 0.18e-5        & 0.84e-6         & 0.80         & 0.27e-4        \\
1000                     & 0.28e-7         & 0.81         & 0.14e-5        & 0.80e-6         & 0.84         & 0.22e-4        \\
2000                     & 0.24e-7         & 0.84         & 0.13e-5        & 0.79e-6         & 0.87         & 0.13e-4        \\
3000                     & 0.21e-7         & 0.88         & 0.12e-5        & 0.79e-6         & 0.88         & 0.12e-4        \\
4000                     & 0.20e-7         & 0.89         & 0.12e-5        & 0.71e-6         & 0.87         & 0.11e-4        \\
1M                       & 0.19e-7         & 0.91         & 0.11e-5        & 0.69e-6         & 0.88         & 0.09e-4        \\ \hline
\end{tabular}
\label{tab:dynamic_model}
\end{table}

As shown in Table~\ref{tab:dynamic_model}, increasing the number of samples used to train the transition model $f$ generally improves RGIA performance across state MSE, RA, and TE.
For Walker2d, RA rises steadily from 0.77 without dynamic consistency regularization to 0.88 with 3,000 samples, while TE decreases from $0.16\times10^{-5}$ to $0.12\times10^{-5}$.
The MSE shows minor fluctuations at small sample sizes but converges to the lowest value ($0.21\times10^{-7}$) with the largest dataset.
For Breakout, RA improves from 0.82 to 0.88, and TE drops markedly from $0.23\times10^{-4}$ to $0.12\times10^{-4}$, with MSE gradually decreasing from $0.83\times10^{-6}$ to $0.79\times10^{-6}$.
These results indicate that larger training sets enable the transition model to more accurately approximate environment transitions, thereby reducing reconstruction error and improving both state and action recovery in RGIA attacks.
However, the improvement in each metric becomes less pronounced when the sample size reaches 3,000.
Even when the transition model is trained on the entire training dataset, there is no significant improvement in MSE, RA, or TE.


\end{document}